\documentclass[twoside,11pt]{article}

\usepackage{jmlr2e}
\usepackage[format=hang]{caption}
\usepackage{amsmath,amsfonts}
\usepackage{mathrsfs}
\usepackage{bbm}
\usepackage{color}
\usepackage{enumerate}
\usepackage{textpos}
\usepackage{float}
\usepackage{epstopdf, subfigure}
\usepackage{pdflscape}     	
\usepackage{ifthen}
\usepackage{cancel}
\usepackage{pdflscape}
\usepackage{array}
\usepackage{url}
\usepackage[usenames,dvipsnames]{xcolor}
\usepackage{xspace}
\usepackage{algorithm}
\usepackage{algorithmic}

\def \one {{\bf 1}}

\newtheorem{defn}{Def}

\def\bd{\begin{defn}}
\def\ed{\end{defn}}
\newtheorem{rem}{Remark}
\def\br{\begin{rem}}
\def\er{\end{rem}}
\def\bp{\begin{prop}}
\def\ep{\end{prop}}
\def\bc{\begin{cor}}
\def\ec{\end{cor}}
\def\bl{\begin{lem}}
\def\el{\end{lem}}

\newcommand{\floor}[1]{\left\lfloor#1\right\rfloor}

\newcommand{\code}[1]{\mbox{\texttt{#1}}}
\newcounter{lineno}
\newenvironment{pseudocode}{\begin{small}\begin{tabbing}\textbf{mm}\=mm\=mm\=mm\=mm\=mm\=mm\=mm\=mm\=\kill}{\end{tabbing}\end{small}}
\newcommand{\codename}{\setcounter{lineno}{0}\>}
\newcommand{\codeline}{\>\stepcounter{lineno}\textbf{\arabic{lineno}}\'\>}

\firstpageno{1}

\begin{document}

\title{Monotonic Calibrated Interpolated Look-Up Tables}

\author{\name Maya Gupta \email mayagupta@google.com \\
\name Andrew Cotter \email acotter@google.com \\
\name Jan Pfeifer \email janpf@google.com \\
\name Konstantin Voevodski \email kvodski@google.com \\
\name Kevin Canini \email canini@google.com \\
\name Alexander Mangylov \email amangy@google.com \\
\name Wojciech Moczydlowski \email wojtekm@google.com \\
\name Alexander van Esbroeck \email alexve@google.com \\
    \addr Google\\
    1600 Amphitheatre Pkwy\\
       Mountain View, CA 94301, USA
       }

\editor{}

\maketitle

\begin{abstract}
Real-world machine learning applications may require functions to be fast-to-evaluate and interpretable.  In particular, guaranteed monotonicity of the learned function with respect to some of the inputs can be critical to user confidence.   We propose meeting these goals for low-dimensional machine learning problems by learning flexible, monotonic functions using calibrated interpolated look-up tables.  We extend the structural risk minimization framework of lattice regression to  train monotonic functions by adding linear inequality constraints. In addition, we propose jointly learning interpretable calibrations of each feature to normalize continuous features and handle categorical or missing data, at the cost of making the objective non-convex.  We address large-scale learning through parallelization, mini-batching, and random sampling of additive regularizer terms.  Case studies on real-world problems with up to sixteen features and up to hundreds of millions of training samples demonstrate the proposed monotonic functions can achieve state-of-the-art accuracy in practice while providing greater transparency to users. 
\end{abstract}


\section{Introduction}
Many challenging issues arise when making machine learning useful in practice. Evaluation of the trained model may need to be fast.  Features may be categorical, missing, or poorly calibrated.  A blackbox model may be unacceptable: users may require guarantees that the function will behave sensibly for all samples, and prefer functions that are easier to understand and debug. 
  
We have found that a key interpretability issue in practice is whether the learned model can be guaranteed to be \emph{monotonic} with respect to some features. For example, suppose the goal is to estimate the value of a used car, and one of the features is the number of km it has been driven. If all the other feature values are held fixed, we expect the value of the used car to never increase as the number of km driven increases. But a model learned from a small set of noisy samples may not, in fact, respect this prior knowledge.

\begin{figure}[t!]
\begin{center}
\begin{tabular}{cccc}
 \includegraphics[width=1.4in]{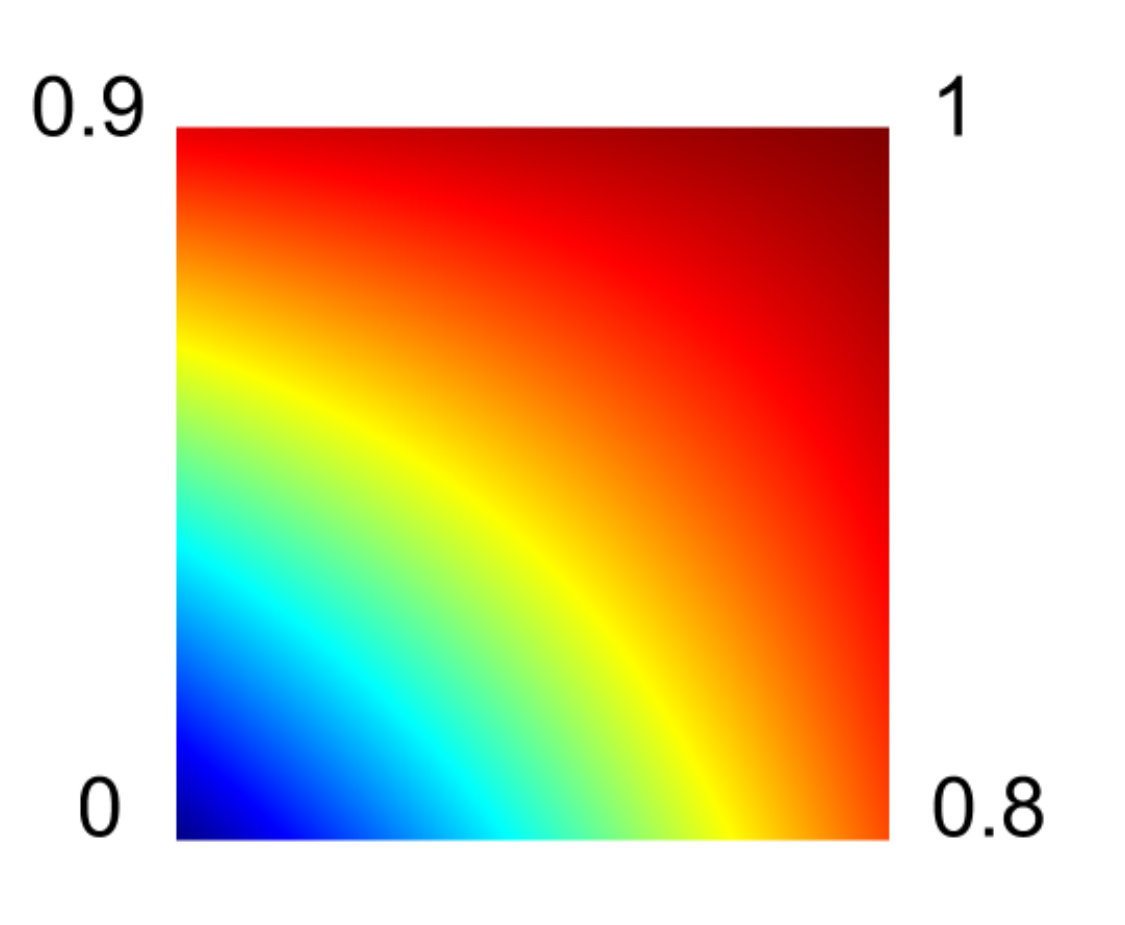} &  \includegraphics[width=1.4in]{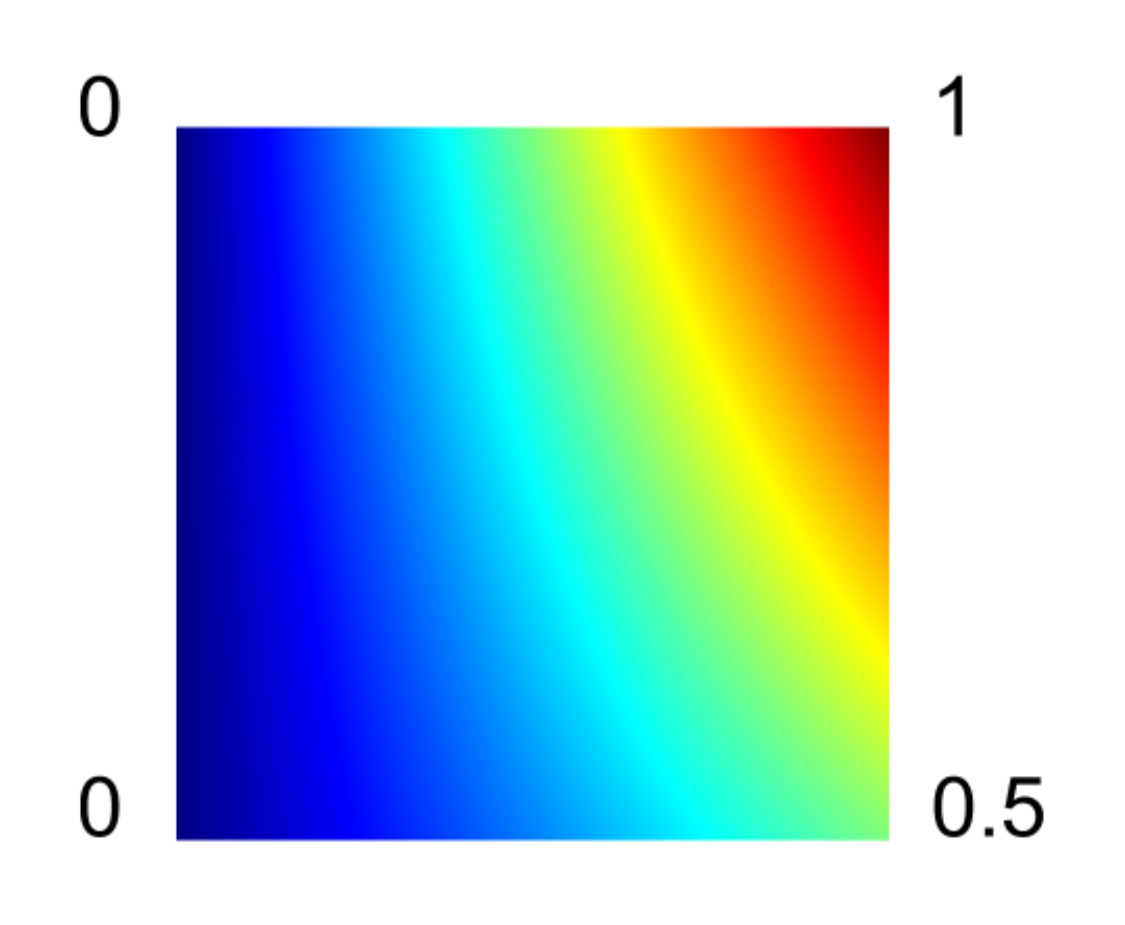}& \includegraphics[width=1.4in]{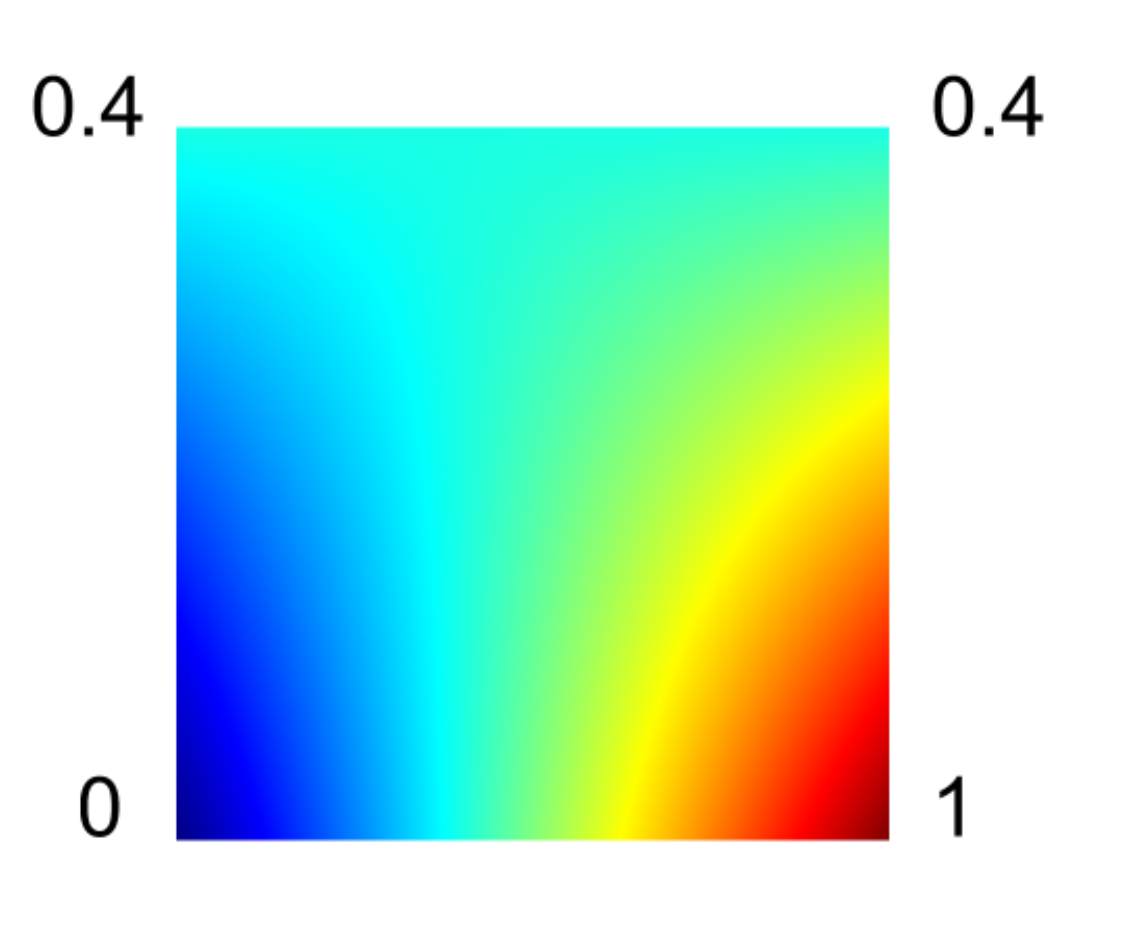}& \includegraphics[width=1.4in]{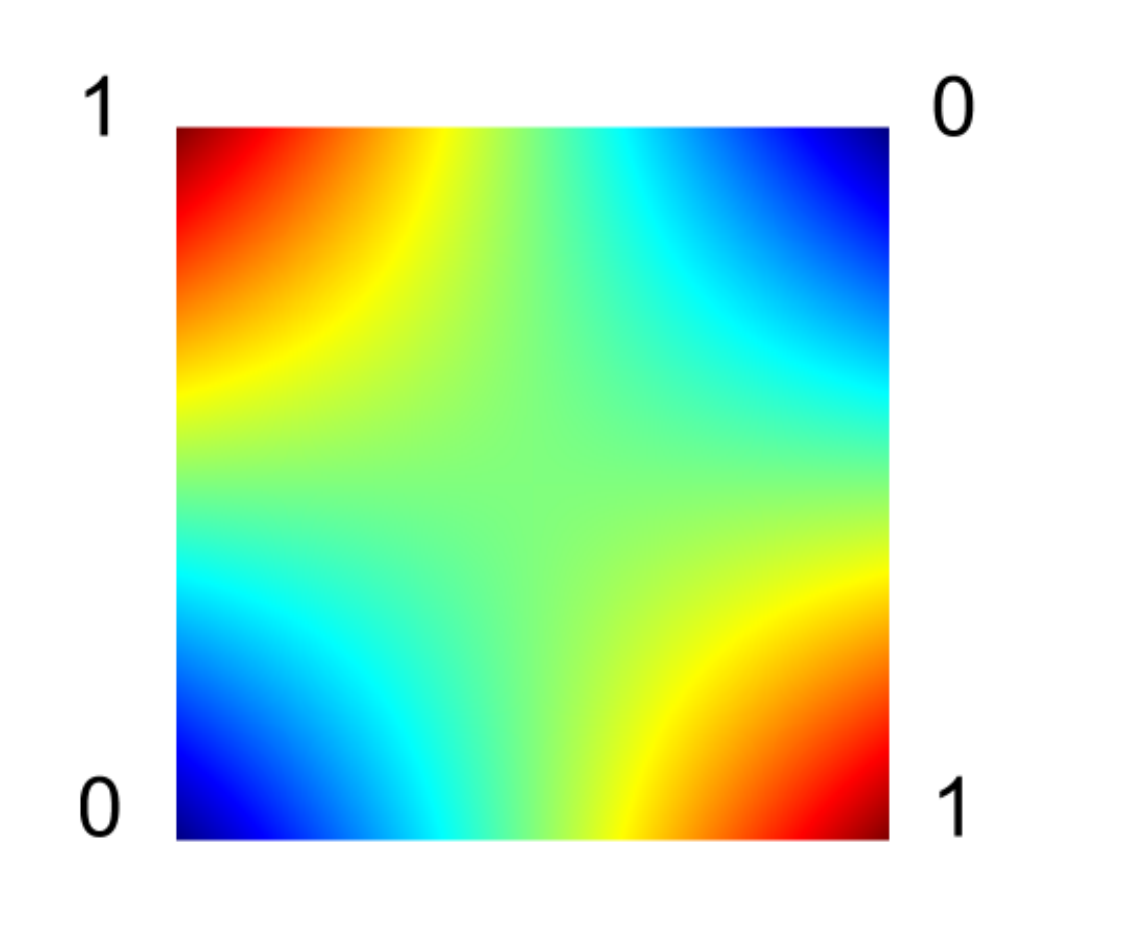} \\
Monotonic & Monotonic & Not Monotonic & Not Monotonic \\
(a) & (b) & (c) & (d) \\
\end{tabular}
\end{center}
\caption{Example $2 \times 2$ interpolated look-up table functions over a unit square. Each function is defined by a  $2 \times 2$ lattice with four parameters, which are the values of the function in the four corners (shown). The function is linearly interpolated from its parameters (see Figure \ref{fig:howItWorks} for a pictorial description of linear interpolation). The function in (a) is strictly monotonically increasing in both features,  which can be verified by checking that each upper parameter is larger than the parameter below it, and that each parameter on the right is larger than the parameter to its left. The  function in (b) is strictly monotonically increasing in the first feature, and monotonically increasing in the second feature (but not strictly so since the parameters on the left are both zero). The function in (c) is monotonically increasing in the first feature (one verifies this by noting that $1 \geq 0$ and $0.4 \geq 0.4$), but non-monotonic in the second feature: on the left side the function increases from $0 \rightarrow 0.4$, but on the right side the function decreases from $1 \rightarrow 0.4$.  The function in (d) is a saddle function interpolating an exclusive-OR, and is non-monotonic in both features.}
\label{fig:exampleLattices}
\end{figure}

In this paper, we propose learning monotonic, efficient, and flexible functions by constraining and calibrating interpolated look-up tables in a structural risk minimization framework. Learning monotonic functions is difficult, and previously published work has only been illustrated on small problems (see Table \ref{tab:relatedWork}).  Our experimental results demonstrate learning flexible, guaranteed monotonic functions on more features and data than prior work, and that these functions achieve state-of-the-art performance on real-world problems.

The parameters of an interpolated look-up table are simply values of the function, regularly spaced in the input space, and these values are interpolated to compute $f(x)$ for any $x$. See Figures \ref{fig:exampleLattices} and \ref{fig:howItWorks} for examples of $2 \times 2$ and $2 \times 3$ look-up tables and the functions produced by interpolating them. Each parameter has a clear meaning:  it is the value of the function for a particular input, for a set of inputs on a regular grid. These parameters can be individually read and checked to understand the learned function's behavior.

Interpolating look-up tables is a classic strategy for representing low-dimensional functions. For example, backs of old textbooks have pages of look-up tables for one-dimensional functions like $sin(x)$, and interpolating look-up tables is standardized by the ICC Profile for the three and four dimensional nonlinear transformations needed to color manage printers \citep{BalaBook}. Using the efficient linear interpolation method we refer to as \emph{simplex interpolation}, we demonstrate that interpolating a look-up table with twenty features took 2 microseconds on a standard CPU. The practical limit to the number of features is more limited by the number of parameters, which scales as $2^D$ for $D$ features.

\begin{figure}[t]
\begin{center}
\begin{tabular}{ccc}
\includegraphics[width=1.9in]{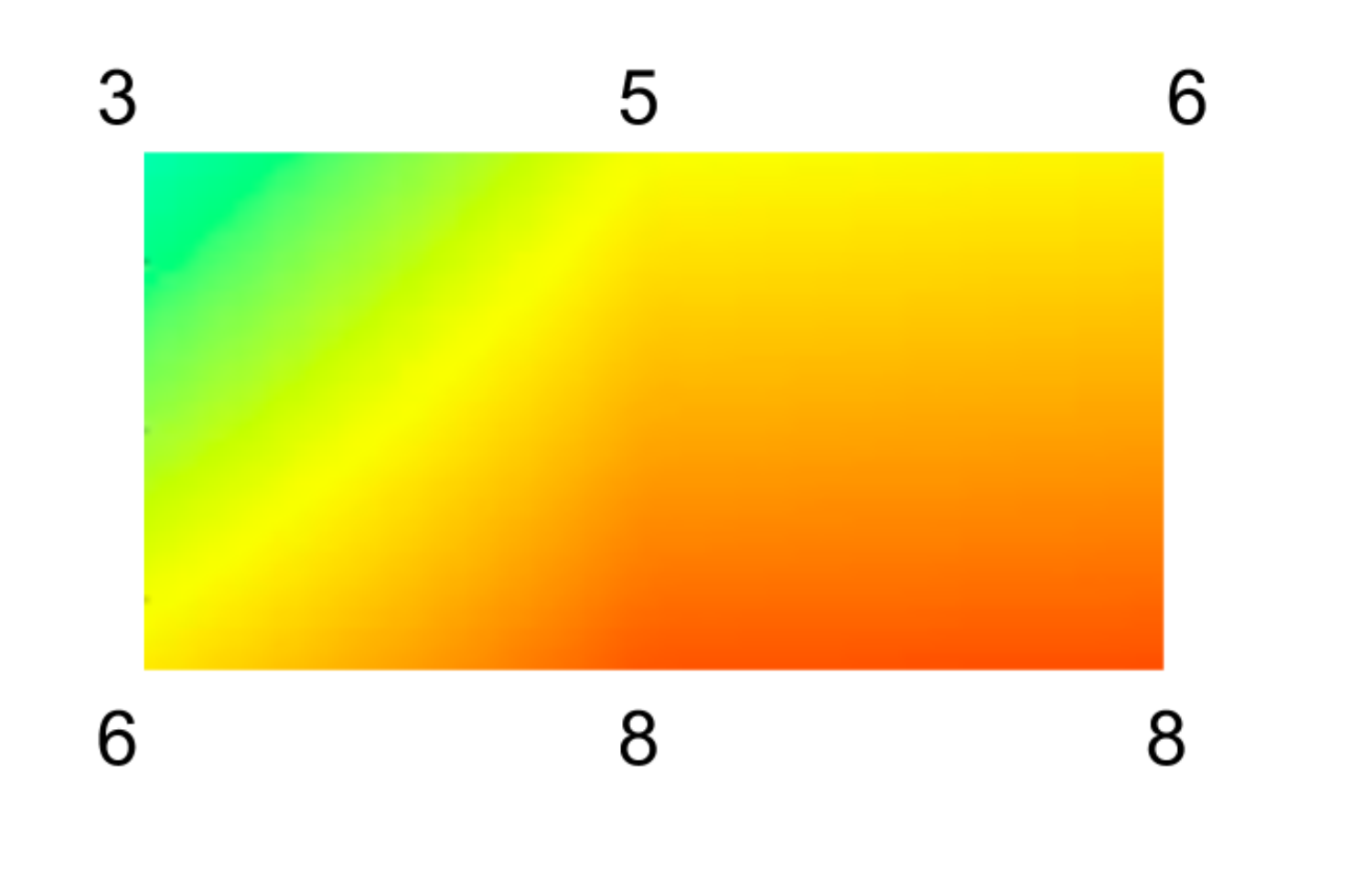} & \includegraphics[width=1.9in]{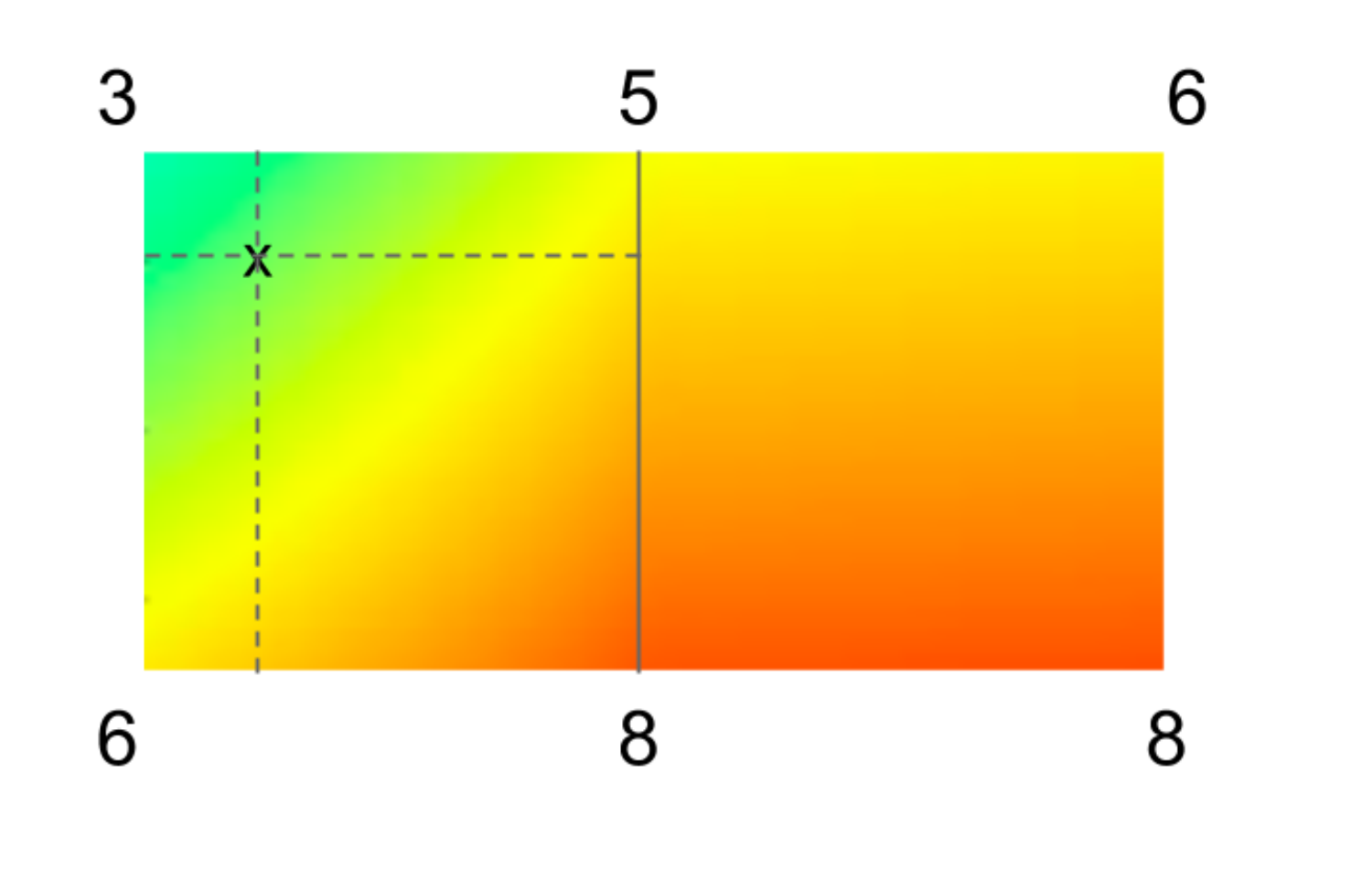}& \includegraphics[width=1.9in]{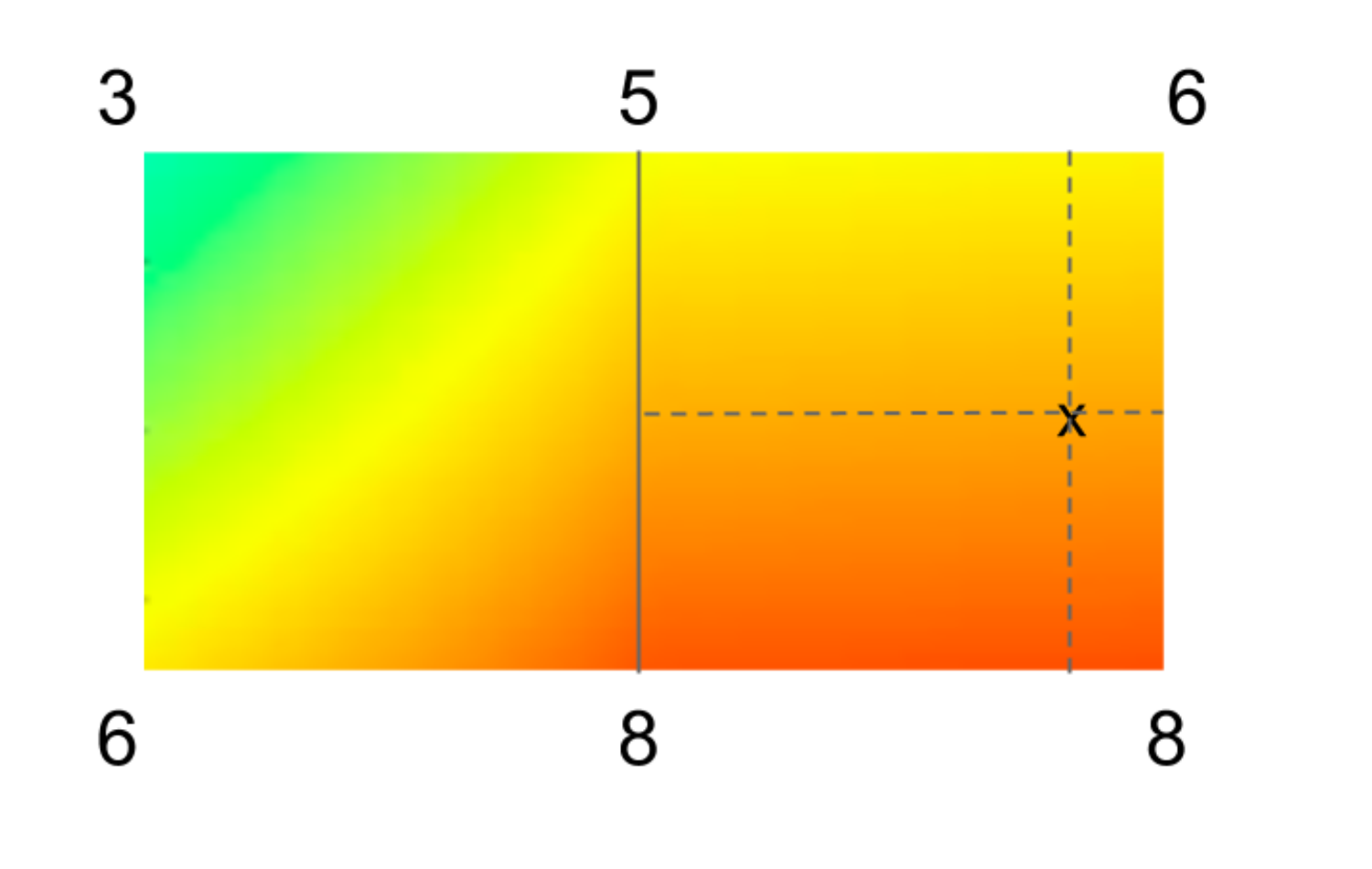} \\
\end{tabular}
\end{center}
\caption{\textbf{Left:} A $3 \times 2$ lattice function with lattice parameters as marked. The function is continuous everywhere, and differentiable everywhere except at the boundary between lattice cells, which is the vertical edge joining the middle parameters 5 and 8. \textbf{Middle:} To evaluate $f(x)$, any $x$ that falls in the left cell of the lattice is linearly interpolated from the parameters at the vertices of the left cell, here 6, 3, 5 and $8$. Linear interpolation is linear not in $x$ but in the lattice parameters, that is $f(x)$ is a weighted combination of the parameter values 6, 3, 5, and 8. The weights on the parameters are the areas of the four boxes formed by the dotted lines drawn orthogonally through $x$, with each parameter weighted by the area of the box farthest from it, so that as $x$ moves closer to a parameter the weight on that parameter gets bigger. Because the dotted lines partition the unit square, the sum of these weights is always one.  \textbf{Right:} Samples $x$ that fall in the right cell of the lattice are interpolated from that cell's parameters: 8, 5, 6 and 8.}
\label{fig:howItWorks}
\end{figure}

Estimating the parameters of an interpolated look-up table using structural risk minimization was proposed by \citet{Garcia:09} and called \emph{lattice regression}. Lattice regression can be viewed as a kernel method that uses the explicit nonlinear feature transformation formed by mapping an input $x \in [0,1]^D$ to a vector of linear interpolation weights $\phi(x) \in \Delta^{2^D}$ over the $2^D$ vertices of the look-up table cell that contains $x$, where $\Delta$ denotes the standard simplex. Then the function is linear in these transformed features: $f(x) = \theta^T\phi(x)$. We will refer to the look-up table parameters $\theta$ as the \emph{lattice}, and to the interpolated look-up table $f(x)$ as the \emph{lattice function}. Earlier work in lattice regression focused on learning highly nonlinear functions over $2$ to $4$ features with fine-grained lattices, such as a $17 \times 17 \times 17$ lattice for modeling a color printer or super-resolution of spherical images \citep{Garcia:10, Garcia:12}. In this paper, we apply lattice regression to more generic machine learning problems with $D = 5$ to $16$ features, and show that $2^D$ lattices work well for many real-world classification and ranking problems, especially when paired with jointly trained one-dimensional pre-processing functions. 

We begin with a survey of related work in  machine learning of interpretable and monotonic functions.  Then we review lattice regression in Section \ref{sec:reviewBasics}. The main contribution is learning monotonic lattices in Section \ref{sec:monotonicity}. We discuss efficient linear interpolation in Section \ref{sec:fasterpussycat}. We propose an interpretable torsion lattice regularizer in Section \ref{sec:regularizers}. We propose jointly learning one-dimensional calibration functions in Section \ref{sec:calibration}, and consider two strategies for supervised handling of missing data for lattice regression in Section \ref{sec:missing}.  In Section \ref{sec:bigData}, we consider strategies for speeding up training and handling large-scale problems and large-scale constraint-handling.  A series of case studies in Section \ref{sec:experiments} experimentally explore the paper's proposals, and demonstrate that monotonic lattice regression achieves similar accuracy as a random forest, and that monotonicity is a common issue that arises in many different applications. The paper ends with conclusions and open questions in Section \ref{sec:discussion}.

\section{Related Work} \label{sec:relatedWork}
We give a brief overview of related work in interpretable machine learning, then survey related work in learning monotonic functions.

\subsection{Related Work in Interpretable Machine Learning}
Two key themes of the prior work on interpretable machine learning are \emph{(i)} interpretable function classes, and \emph{(ii)} preferring simpler functions within a function class. 

\subsubsection{Interpretable Function Classes}
The function classes of decision trees and rules are generally regarded as relatively interpretable. Na{\"i}ve Bayes classifiers can be interpreted in terms of weights of evidence \citep{Good:65,Spiegelhalter:84}. Similarly, linear models form an interpretable function class in that the parameters dictate the relative importance of each feature.  Linear approaches can be generalized to sum nonlinear components, as in generalized additive models \citep{HastieTibshirani:90} and some kernel methods, while still retaining some of their interpretable aspects. 


The function class of interpolated look-up tables is interpretable in that the function's parameters are the look-up table values, and so are semantically meaningful: they are simply examples of the function's output, regularly spaced in the domain. Given two look-up tables with the same structure and the same features, one can analyze how their functions differ by analyzing how the look-up table parameters differ. Analyzing which parameters change by how much can help answer questions like ``If I add training examples and re-train, what changes about the model?"

\subsubsection{Prefer Simpler Functions}
Another body of work focuses on choosing simpler functions within a function class, optimizing an objective of the form: minimize empirical error and maximize simplicity, where simplicity is usually defined as some manifestation of Occam's Razor or variant of Kolmogorov complexity. For example, \citet{Ishibuchi:07} minimize the number of fuzzy rules in a rule set,  \citet{Osei:07} prunes a decision tree for  interpretability,  \citet{Ratsch:06} finds a sparse convex combination of kernels for a multi-kernel support vector machine, and \citet{Nock:02} prefers smaller committees of ensemble classifiers. Similarly, \citet{Herrera:09} measure the interpretability of rule-based classifiers in terms of the number of rules and number of features used. More generally, this category of interpretability includes model selection criteria like the Bayesian information criterion and Akaike information criterion \citep{HTF}, sparsity regularizers like sparse linear regression models, and feature selection methods. Other approaches to simplicity may include simplified structure in graphical models or neural nets, such as the structured neural nets of \citep{Strannegaard:12}. 

While sparsity-based approaches to interpretability can provide regularization that reduces over-fitting and hence increases accuracy, it has also been noted that such strategies may create a trade-off between interpretability and accuracy \citep{Casillas:02,Nock:02,Yu:09,Shukla:12}. We hypothesize this occurs when the assumed simpler structure is a poor model of the true function. 

Monotonicity is another way to choose a semantically simpler function to increase interpretability (and regularize). Our case studies in Section \ref{sec:experiments} illustrate that when applied to problems where monotonicity is a good prior model, we do not see a trade-off with accuracy.  



\subsection{Related Work in Monotonic Functions}\label{sec:relatedMonotonic}
A function $f(x)$ is monotonically increasing with respect to feature $d$ if $f(x_i) \geq f(x_j)$ for any two feature vectors $x_i, x_j \in \mathbb{R}^D$ where $x_i[d] \geq x_j[d]$ and $x_i[m] = x_j[m]$ for $m \neq d$.

A number of approaches have been proposed for enforcing and encouraging monotonicity in machine learning. The computational complexity of these algorthims tends to be high, and most methods scale poorly in the number of features $D$ and samples $n$, as summarized in Table \ref{tab:relatedWork}. \\

We detail the related work in the following sections organized by the type of machine learning, but these methods could instead be organized by strategy, which mostly falls into one of four categories:

\begin{enumerate}
\item Constrain a more flexible function class to be monotonic, such as linear functions with positive coefficients, or a sigmoidal neural network with positive weights.

\item Post-process by pruning or reducing monotonicity violations after training. 

\item Penalize monotonicity violations by pairs of samples or sample derivatives when training.

\item Re-label samples to be monotonic before training.

\end{enumerate}

\subsubsection{Monotonic Linear and Polynomial Functions}
Linear functions can be easily constrained to be monotonic in certain inputs by requiring the corresponding slope coefficients to be non-negative, but linear functions are not sufficiently flexible for many problems. Polynomial functions (equivalently, linear functions with pre-defined crosses of features) can also be easily forced to be monotonic by requiring all coefficients to be positive. However, this is only a sufficient and not necessary condition: there are monotonic polynomials whose coefficients are not all positive. For example, consider the simple case of second degree multilinear polynomials defined over the unit square $f:[0,1]^2\rightarrow \mathbb{R}$ such that:
\begin{equation}
f(x) = a_0  + a_1 x[0] + a_2 x[1] + a_3 x[0] x[1]. \label{eqn:simplepoly}
\end{equation}
Forcing the derivative to be positive on the domain $x \in [0,1]^2$, one sees that the complete set of monotonic functions of the form (\ref{eqn:simplepoly}) on the unit square is described by four linear inequalities:
\begin{align*}
a_1 &> 0 \\
a_2 &> 0 \\
a_1 + a_3  & > 0 \\
a_2 + a_3  & > 0.
\end{align*}
The general problem of checking whether a particular choice of polynomial coefficients produces a monotonic function requires checking whether the polynomial's derivative (also a polynomial) is positive everywhere, which is equivalent to checking if the derivative has any real roots, which can be computationally challenging (see, for example, Sturm's theorem for details).

Functions of the form (\ref{eqn:simplepoly}) can be equivalently expressed as a $2 \times 2$ lattice interpolated with multilinear interpolation, but as we will show in Section \ref{sec:monotonicity},  with this alternate parameterization it is easier to check and enforce the complete set of monotonic functions.

\subsubsection{Monotonic Splines}
In this paper we extend lattice regression, which is a spline method with fixed knots on a regular grid and a linear kernel  \citep{Garcia:12}, to be monotonic.  There have been a number of proposals to learn monotonic one-dimensional splines. For example, building on \citet{Ramsay:98}, \citet{Shively:09} parameterize the set of all smooth and strictly monotonic one-dimensional functions using an integrated exponential form $f(x) = a + \int_0^x e^{b + u(t)}dt$, and showed better performance than the monotone functions estimators of \citet{Neelon:04} and \citet{Holmes:03} for smooth functions.  In other related spline work, \citet{Wahba:87} considered smoothing splines with linear inequality constraints, but did not address monotonicity. 


\subsubsection{Monotonic Decision Trees and Forests:}
Stumps and forests of stumps are easily constrained to be monotonic. However, for deeper or broader trees, all pairs of leaves must be checked to verify monotonicity \citep{Pardoel:02}. Non-monotonic trees can be pruned to be monotonic using various strategies that iteratively reduce the non-monotonic branches \citep{BenDavid:92,Pardoel:02}. Monontonicity can also be encouraged during tree construction by penalizing the splitting criterion to reduce the number of non-monotonic leaves a split would create \citep{benDavid:95}.  \citet{Feelders:02} achieved completely flexible monotonic trees using a strategy akin to bogosort \citep{bogosort}: train many trees on different random subsets of the training samples,  then select one that is monotonic.

\begin{landscape}

\begin{table}
\begin{tabular}{l|llcrr}
& Method & Monotonicity Strategy & Guaranteed Monotonic?   &  Max  $D$ & Max $n$   \\ \hline
\citet{Archer:93} & neural net& constrain function & yes  & 2 & 50 \\
\citet{wang:94} & neural net& constrain function & yes  & 1 & 150  \\
\citet{Mukarjee:94}  & kernel estimate & post-process & yes  & 2 & 2447 \\
\citet{benDavid:95} & tree & penalize splits & yes   & 8 & 125  \\
\citet{Sill:97} & neural net & penalize pairs & no  & 6 &  550  \\
\citet{Sill:98} & neural net & constrain function & yes  & 10 & 196 \\
\citet{KayUngar:2000} & neural net & constrain function & yes  & 1 & 100 \\
\citet{Feelders:02} & tree & randomize & yes & 8 & 60  \\
\citet{Pardoel:02} & tree & prune & yes  & 11 & 1090  \\
\citet{Spouge:03} & isotonic regression & constrain & yes  & 2 & 100,000 \\
\citet{Feelders:08} & k-NN & re-label samples & no  & 12 &  768 \\
\citet{lauer:2008r} & svm & sample derivatives  & no  & none & none  \\
\citet{Dugas:NIPS,Dugas:JMLR} & neural net & constrain function & yes  & 4 & 3434  \\
\citet{Shively:09} & spline & constrain function & yes  & 1 & 100 \\
\citet{Kotlowski:09} & rule-based & re-label samples & yes & 11 & 1728 \\
\citet{Daniels:2010} &  neural net & constrain function & yes & 6 & 174 \\
\citet{riihimaki:2010} & Gaussian process &  sample derivatives & no  & 7 & 1222 \\
\citet{QuHu:11} & neural net & derivatives / constrain & yes & 1 & 30  \\
\citet{neumann:2013} & neural net & sample derivatives & no & 3 & 625   \\
\end{tabular}
\caption{Some related work in learning monotonic functions. Many of these methods \emph{guarantee} a monotonic solution, but some only encourage monotonicity. The last two columns gives the largest number of features $D$ and the largest number of samples $n$ used in any of the experiments in that paper (generally not the same experiment). }\label{tab:relatedWork}
\end{table}
\end{landscape}

\subsubsection{Monotonic Support Vector Machines}
With a linear kernel, it may be easy to check and enforce monotonicity of support vector machines, but for nonlinear kernels it will be more challenging.  \citet{lauer:2008r}  encouraged support vector machines to be more monotonic by constraining the derivative of the function at the training samples. \citet{riihimaki:2010} used the same strategy to encourage more monotonic Gaussian processes.

\subsubsection{Monotonic Neural Networks}\label{sec:monotonicNeuralNets}
In perhaps the earliest work on monotonic neural networks, \citet{Archer:93} adaptively down-weighted samples during training whose gradient updates would violate monotonicity, to produce a positive weighted neural net. Other researchers explicitly proposed constraining the weights to be positive in a single hidden-layer neural network with the sigmoid or other monotonic nonlinear transformation  \citep{wang:94,KayUngar:2000,Dugas:NIPS,Dugas:JMLR,Minin:2010}. \citet{Dugas:JMLR} showed with simulations of four features and 400 training samples that both bias and variance were reduced by enforcing monotonicity. However, \citet{Daniels:2010} showed this approach requires $D$ hidden layers to arbitrarily approximate any $D$-dimensional monotonic function. In addition to a general proof, they provide a simple and realistic example of a two-dimensional monotonic function that cannot be fit with one hidden layer and positive weights. 

\citet{AbuMostafa:93} and \citet{Sill:97} proposed regularizing a function to be more monotonic by penalizing squared deviations in monotonicity for pairs of the input samples. This strategy only works if all the features are constrained to be monotonic (otherwise it is not clear how to order a given pair of input samples). Unfortunately, it generally does not guarantee monotonicity everywhere, only with respect to those sampled pairs. (And in fact, to guarantee monotonicity for the sampled pairs, an exact penalty rather than squared error would be needed with a sufficiently large regularization parameter to ensure the regularization was equivalent to a constraint). 

\citet{lauer:2008r}, \citet{riihimaki:2010}, and \citet{neumann:2013} encouraged \emph{extreme learning machines} to be more monotonic by constraining the derivative of the function to be positive for a set of sampled points.

\citet{QuHu:11} did a small-scale comparison of encouraging monotonicity by constraining input pairs to be monotonic, encouraging monotonic neural nets by constraining the function's derivatives at a subset of samples (analogous to \citet{lauer:2008r}), and using a sigmoidal function with positive weights.  They concluded the positive-weight sigmoidal function is best. 

\citet{Sill:98} proposed a guaranteed monotonic neural network with two hidden layers by requiring the first linear layer's weights to be positive, using hidden nodes that take the maximum of groups of first layer variables, and a second hidden layer that takes the minimum of the maxima. The resulting surface is piecewise linear, and as such can represent any continuous differentiable function arbitrarily well. The resulting objective function is not strictly convex, but the authors propose training such monotonic networks using gradient descent where samples are associated with one active hyperplane at each iteration. \citet{Daniels:2010} generalized this approach to handle the ``partially monotonic" case that the function is only monotonic with respect to some features.

\subsubsection{Isotonic Regression and Monotonic Nearest Neighbors}
Isotonic regression re-labels the input samples with values that are monotonic and close to the original labels. These monotonically re-labeled samples can then be used, for example, to define a monotonic piecewise constant or piecewise linear surface. This is an old approach; see \citet{Barlow:72} for an early survey. Isotonic regression can be solved in $O(n)$ time if monotonicity implies a total ordering of the $n$ samples. But for usual multi-dimensional machine learning problems, monotonicity implies only a partial order, and solving the $n$-parameter quadratic program is generally $O(n^4)$, and $O(n^3)$ for two-dimensional samples \citep{Spouge:03}.   Also problematic for large $n$ is the $O(n)$ evaluation time for new samples.


\citet{Mukarjee:94} proposed a suboptimal monotonic kernel regression that is computationally easier to train than isotonic regression. It computes a standard kernel estimate, then locally upper and lower bounds it to enforce monotonicity, for overall $O(n)$ evaluation time.
 
The \emph{isotonic separation} method of \citet{Chandrasekaran:05} is like the work of \citet{AbuMostafa:93} in that it penalizes violations of monotonicity by pairs of training samples.  Like isotonic regression, the output is a re-labeling of the original samples, the solution is at least $O(n^3)$ in the general case, and evaluation time is $O(n)$.

\citet{BenDavid:89,BenDavid:92} constructed a monotonic rule-based classifier by sequentially adding training examples (each of which defines a rule) that do not violate monotonicity restrictions.

\citet{Feelders:08} proposed re-labeling samples before applying nearest neighbors based on a monotonicity violation graph with the training examples at the vertices. Coupled with a proposed modified version of k-NN, they can enforce monotonic outputs. Similar  pre-processing of samples can be used to encourage any subsequently trained classifier to be more monotonic \citep{Feelders:10}.

Similarly, \citet{Kotlowski:09} try to solve the isotonic regression problem to re-label the dataset to be monotonic, then fit a monotonic ensemble of rules to the re-labeled data, requiring zero training error.  They showed overall better performance than the ordinal learning model of \citet{BenDavid:89} and isotonic separation \citep{Chandrasekaran:05}.






\section{Review of Lattice Regression}\label{sec:reviewBasics} 
Before proposing \emph{monotonic} lattice regression, we review lattice regression \citep{Garcia:09,Garcia:12}. Key notation is listed in Table \ref{tab:notation}.

Let  $M_d \in \mathbb{N}$ be a hyperparameter specifying the number of vertices in the look-up table (that is, lattice) for the $d$th feature. Then the lattice is a regular grid of $M  \stackrel{\triangle}{=} M_1 \times M_2 \times \ldots M_D$ parameters (a look-up table) placed at natural numbers so that the lattice spans the hyper-rectangle $\mathcal{M} \stackrel{\triangle}{=} [0, M_1 - 1] \times [0, M_2 - 1] \times \ldots [0, M_D - 1]$. See Figure \ref{fig:exampleLattices} for examples of  $2 \times 2$ lattices, and Figure \ref{fig:howItWorks} for an example $3 \times 2$ lattice.  For machine learning problems we find $M_d = 2$ for all $d$ to be a good default, as detailed in the case studies in Section \ref{sec:experiments}. For image processing applications with only two to four features, much larger values of $M_d$ were needed \citep{Garcia:12}. 

The feature values are assumed to be bounded and linearly scaled to fit the lattice, so that the $d$th feature vector value $x[d]$ lies in $[0, M_d - 1]$. (We propose learning non-linear scalings of features jointly with the lattice parameters in Section \ref{sec:calibration}.)

\begin{table}[hb!]
\begin{center}
\begin{tabular}{ll}
$D$ & number of features \\
$n$ & number of samples \\
$M_d \in \mathbb{N}$ & number of vertices in the lattice along the $d$th feature \\
$M \in \mathbb{N}$ & total number of vertices in the lattice: $M=\prod_d M_d$ \\
$\mathcal{M}$ & hyper-rectangular span of the lattice: $[0, M_1 - 1] \times \ldots \times  [0, M_D - 1]$  \\
$x_i$ & $i$th training sample with $D$ components. Domain depends on section. \\
$y_i \in \mathbb{R}$ & $i$th training sample label \\
$x[d] $ & $d$th component of feature vector $x$ \\ 
$\phi(x) \in [0,1]^{M}$ & linear interpolation weights for $x$ \\
$\theta \in \mathbb{R}^M$ & lattice values (parameters) \\
$v_j \in \{0,1\}^D$ & $j$th vertex of a $2^D$ lattice
\end{tabular}
\end{center}
\caption{Key Notation}
\label{tab:notation}
\end{table}

Lattice regression is a kernel method that maps $x \in \mathcal{M}$ to a transformed feature vector  $\phi(x) \in [0,1]^{M}$.  The values of $\phi(x)$ are the interpolation weights for $x$ for the $2^D$ indices corresponding to the $2^D$ vertices of the hypercube surrounding $x$; for all other indices, $\phi(x) = 0$.

The function $f(x)$ is linear in $\phi(x)$ such that $f(x) = \theta^T \phi(x)$.  That is, the function parameters $\theta$ each correspond to a vertex in the lattice, and $f(x)$ linearly interpolates the $\theta$ for the lattice cell containing $x$.

Before reviewing the lattice regression objective for learning the parameters $\theta$, we review standard multilinear interpolation to define $\phi(x)$.

\subsection{Multilinear Interpolation}\label{sec:multilinearInterpolation}
The familiar bilinear interpolation commonly used to up-sample images is the $D=2$ case of the multilinear interpolation that we review here.   See Figure \ref{fig:howItWorks} for a pictorial explanation. 

For notational simplicity, we assume a $2^D$ lattice such that $x \in [0,1]^D$. For multi-cell lattices, the same math and logic is applied to the lattice cell containing the $x$.  Denote the $k$th component of $\phi(x)$ as $\phi_k(x)$. Let $v_k \in \{0,1\}^D$ be the $k$th vertex of the unit hypercube. The multilinear interpolation weight on the vertex $v_k$ is
\begin{equation}\label{eqn:weight}
\phi_k(x)  = \prod_{d = 0}^{D-1} x[d]^{v_k[d]} (1 - x[d])^{1-v_k[d]}.
\end{equation}
Note the exponents in (\ref{eqn:weight}) are $v_k$ and $1-v_k[d]$, which either equal $0$ and $1$, or  equal $1$ and $0$, so these exponents act like selectors that multiply in either $x[d]$ or $1-x[d]$ for each dimension $d$.  Equivalently, one can write
\begin{equation}
  \phi_k(x) = \prod_{i=0}^{D-1} \left(
    \left(1 - \code{bit}[i,k]\right) \left(1 - x[i] \right) +
  \code{bit}[i,k] x[i] \right),\label{eqn:bitwise}
\end{equation}
where $\code{bit}[i,k] \in \{0,1\}$ denotes the $i$th bit of vertex $v_k$, and can be computed 
$\code{bit}[i,k] = \left(k \gg i\right) \& 1$ using bitwise arithmetic.

The resulting $f(x) = \theta^T \phi(x)$ is a multilinear polynomial over each lattice cell. For example, a $2 \times 2$ lattice interpolated with multilinear interpolation  (\ref{eqn:weight})  produces the function:
\begin{equation}\label{eqn:long}
f(x) = \theta[0] (1 - x[0]) (1 - x[1]) + \theta[1] x[0](1 - x[1]) + \theta[2] (1 - x[0]) x[1]  + \theta[3] x[0] x[1].
\end{equation}
Expanding (\ref{eqn:long}), one sees it is  a different parameterization of the multilinear function given in (\ref{eqn:simplepoly}), where the parameter vectors are related by a linear matrix transform:  $a = T\theta$ for $T \in \mathbb{R}^{4 \times 4}$. But parameterizing with $\theta$ makes the parameters easier to read, and as we show in Section \ref{sec:monotonicity}, makes it easier to learn the complete set of monotonic functions.  

The linear interpolation is applied per lattice cell. At lattice cell boundaries the resulting function is continuous, but not differentiable, and is piecewise polynomial, and hence a spline. It can be equivalently formulated using a linear basis function. Higher-order basis functions like the popular cubic spline will lead to smoother and potentially slightly more accurate functions \citep{Garcia:12}. However, higher-order basis functions destroy the interpretable localized effect of the parameters, and increase the computational complexity. 

The multilinear interpolation weights are just one type of linear interpolation. In general, linear interpolation weights are defined as solutions to the  system of $D+1$ equations:
\begin{equation}
 \sum_{k=0}^{2^D}  \phi_k(x) v_k  = x \textrm{ and } \sum_{k=0}^{2^D} \phi_k(x) = 1. \label{eqn:linInterp}
\end{equation}
This system of equations is under-determined and has many solutions for an $x$ in the convex hull of a lattice cell. The multilinear interpolation weights given in (\ref{eqn:weight}) are the maximum entropy solution to (\ref{eqn:linInterp})  \citep{GuptaGrayOlshen}, and thus have good noise averaging and smoothness properties compared to other solutions.  We discuss a more efficient linear interpolation in Sec. \ref{sec:simplex}.

\subsection{The Lattice Regression Objective}\label{sec:objective}
Consider the standard supervised machine learning set-up of a training set of randomly sampled pairs $\{(x_i, y_i)\}$ pairs, where $x_i \in \mathcal{M}$ and $y_i \in \mathbb{R}$, for $i = 1, \ldots, n$. Historically, people created look-up tables by first fitting a function $h(x)$ to the $\{x_i, y_i\}$ using a regression algorithm such as a  neural net or local linear regression, and then evaluating $h(x)$ on a regular grid to produce the look-up table values \citep{BalaBook}. However, even if they fit the function to minimize empirical risk on the training samples, they did not minimize the \emph{actual} empirical risk because these approaches did not take into account that the produced look-up table would be interpolated at run-time, and this interpolation changes the error on the training samples. 

\citet{Garcia:09} proposed directly optimizing the look-up table parameters $\theta$ to minimize the empirical error between the training labels and the interpolated look-up table:
\begin{align} \label{eqn:latticeRegression}
\arg \min_{\theta \in \mathbb{R}^M} \sum_{i=1}^n \ell (y_i, \theta^T \phi(x_i)) + R(\theta),
\end{align}
where $\ell$ is a loss function such as squared error, $\phi(x_i) \in [0,1]^{M}$ is the vector of linear interpolation weights over the lattice for training sample $x_i$ (detailed in Section \ref{sec:multilinearInterpolation} and Sec. \ref{sec:simplex}), $f(x_i) = \theta^T\phi(x_i)$ is the linear interpolation of $x_i$ from the look-up table parameters $\theta$, and $R(\theta)$ is a regularizer on the lattice parameters. In general, we assume the loss $\ell$ and regularizer $R$ are convex functions of $\theta$ so that solving (\ref{eqn:latticeRegression}) is a convex optimization. Prior work focused on squared error loss, and used graph regularizers $R(\theta)$ of the form $b^TKb$ for some PSD matrix $K$, in which case (\ref{eqn:latticeRegression}) has a closed-form solution which can be computed with sparse matrix inversions \citep{Garcia:09,Garcia:10,Garcia:12}.


\section{Monotonic Lattices}\label{sec:monotonicity}
In this section we propose constraining lattice regression to learn monotonic functions.

\subsection{Monotonicity Constraints For a Lattice}
In general, simply checking whether a nonlinear function is monotonic can be quite difficult (see the related work in Section \ref{sec:relatedMonotonic}). But for a linearly interpolated look-up table, checking for monotonicity is relatively easy: if the lattice values increase in a given direction, then the function increases in that direction. See Figure \ref{fig:exampleLattices} for examples.  Specifically, one must check that $\theta_s > \theta_r$ for each pair of adjacent look-up table parameters $\theta_r$ and $\theta_s$. If all features are specified to be monotonic for a $2^D$ lattice, this results in $D2^{D-1}$ pairwise linear inequality constraints to check. 

These same pairwise linear inequality constraints can be imposed when learning the parameters $\theta$ to ensure a monotonic function is learned.  The following result  establishes these constraints are sufficient and necessary for a $2^D$ lattice to be monotonically increasing in the $d$th feature (the result extends trivially to larger lattices):

\begin{lemma}[Monotonicity Constraints] Let $f(x)= \theta^T\phi(x)$ for $x \in [0,1]^D$ and $\phi(x)$ given in (\ref{eqn:weight}). The partial derivative $\partial f(x)/\partial x[d] >0$ for fixed $d$ and any $x$ iff $\theta_{k'} > \theta_{k}$ for all $k,k'$ such that $v_k[d] =0$, $v_{k'}[d] = 1$ and $v_k[m] = v_{k'}[m]$ for all $m \neq d$. 
\end{lemma}

\begin{proof}
First we show the constraints are necessary to ensure monotonicity. Consider the function values $f(v_k)$ and $f(v_{k'})$ for some adjacent pair of vertices $v_k, v_{k'}$ that differ only in the $d$th feature. For  $f(v_k)$ and $f(v_{k'})$, all of the interpolation weight falls on $\theta_k$ or $\theta_{k'}$ respectively, such that $f(v_k) = \theta_k$ and $f(v_{k'}) = \theta_{k'}$. So $\theta_{k'} > \theta_k$ is necessary for $\partial f(x)/\partial x[d]  > 0$ everywhere.

Next we show the constraints are sufficient to ensure monotonicity. Pair the terms in the interpolation $f(x) =  \theta^T \phi(x)$ corresponding to adjacent parameters $\theta_k, \theta_{k'}$ so that for each $k, k'$ it holds that $v_k[d] = 0, v_{k'}[d] = 1, v_k[m] = v_{k'}[m]$ for $m \neq d$:
\begin{align}
f(x) &= \sum_{k, k'} \theta_k \phi_k(x) + \theta_{k'} \phi_{k'}(x), \textrm{ then expand } \phi_k(x) \textrm{ and } \phi_{k'}(x) \textrm{ using } (\ref{eqn:weight}): \nonumber \\
&= \sum_{k, k'} \alpha_k \left(\theta_k x[d]^{v_k[d]}(1-x[d]^{(1-v_k[d])}) + \theta_{k'} x[d]^{v_{k'}[d]}(1-x[d]^{(1-v_{k'}[d])})\right), \nonumber \\ 
& \textrm{ where } \alpha_k \textrm{ is the product of the $m \neq d$ terms in (\ref{eqn:weight}) that are the same for $k$ and $k'$}, \nonumber \\
&= \sum_{k, k'}  \alpha_k \left(\theta_k (1-x[d]) + \theta_{k'} x[d]\right)  \textrm{ by the definition of $v_k$ and $v_{k'}$.} \label{eqn:hippos} 
\end{align}
The partial derivative of (\ref{eqn:hippos}) is $\frac{\partial f(x)}{\partial x[d]} = \sum_{k, k'} \alpha_k(\theta_{k'} - \theta_k)$. Because each $ \alpha_k \in [0,1]$,  it is sufficient that $\theta_{k'} > \theta_k$ for each $k,k'$ pair to guarantee this partial is positive for all $x$.
\end{proof}

\subsection{Monotonic Lattice Regression Objective}
We relax strict monotonicity to monotonicity by allowing equality in the adjacent parameter constraints (for an example, see the second function from the left in Figure \ref{fig:exampleLattices}). Then the set of pairwise constraints can be expressed as $A \theta \leq 0$ for the appropriate sparse matrix $A$ with one $1$ and $-1$ per row of $A$, and one row per constraint. Each feature can independently be left unconstrained, or constrained to be either monotonically increasing or decreasing by the specificiation of $A$. 

Thus the proposed monotonic lattice regression objective is convex  with linear inequality constraints:
\begin{equation} \label{eqn:latticeRegressionMonotonic}
\arg \min_{\theta} \sum_{i=1}^n \ell (y_i, \theta^T \phi(x_i)) + R(\theta), \textrm{ s.t. } A \theta \leq b.
\end{equation}
Additional linear constraints can be included in $A \theta \leq b$ to also constrain the fitted function in other practical ways, such as  $f(x) \in [0,1]$ or $f(x) \geq 0$.

The approach extends to the standard learning to rank from pairs problem \citep{LiuBook}, where the training data is pairs of samples $x_i^+$ and $x_i^-$ and the goal is to learn a function such that $f(x_i^+) \geq f(x_i^-)$ for as many pairs as possible. For this case, the monotonic lattice regression objective is:
\begin{equation} \label{eqn:latticeRegressionRanking}
\arg \min_{\theta} \sum_{i=1}^n \ell (1, \theta^T \phi(x_i^+) - \theta^T \phi(x_i^-)) + R(\theta),  \textrm{ s.t. } A \theta \leq b.
\end{equation}
The loss functions in (\ref{eqn:latticeRegression}), (\ref{eqn:latticeRegressionMonotonic}) and (\ref{eqn:latticeRegressionRanking}) all have the same form, for example, squared loss $\ell(y,z) = (y -
z)^2$, hinge loss $\ell(y, z) = \max(0, 1 - yz)$, or logistic loss $\ell(y,z) = \log(1 + \exp(y - z))$. 

\section{Faster Linear Interpolation}\label{sec:fasterpussycat}
Interpolating a look-up table has long been considered an efficient way to specify and evaluate a low-dimensional non-linear function \citep{BalaBook,Garcia:12}.  But computing linear interpolation weights with (\ref{eqn:bitwise}) requires $O(D)$ operations for each of the $2^D$ interpolation weights, for a total cost of $O(D2^D)$. In Section \ref{sec:fast}, we show the multilinear interpolation weights of (\ref{eqn:bitwise}) can be computed in $O(2^D)$ operations. Then, in Section \ref{sec:simplex}, we review and analyze a different linear interpolation that we refer to as \emph{simplex} interpolation that takes only $O(D \log D)$ operations.

\subsection{Fast Multilinear Interpolation}\label{sec:fast}
Much of the computation in (\ref{eqn:bitwise}) can be shared between the different weights.  In Algorithm \ref{alg:dynamic-programming} we give a dynamic programming solution that loops $D$ times, where the $d$th loop takes $2^d$ time,  so in total there are $\sum_{d=0}^{D-1} 2^d = O(2^D)$ operations.  

\begin{algorithm*}
\begin{pseudocode}
\codename $\code{CalculateMultilinearInterpolationWeightsAndParameterIndices}(x)$\\
\codeline Initialize $\code{indices}= [0]$, $\code{weights} = [1]$\\
\codeline For $d = 0$ to $D - 1$:   \\
\codeline \>For $k = 0$ to $2^d - 1$:   \\ 
\codeline \>\> Append $s_d + \code{indices}[k]$ to \code{indices} \\
\codeline \>\> Append $x[d]  \times \code{weights}[k]$ to $\code{weights}$ \\
\codeline \>\>Update $\code{weights}[k] = \left(1-( x[d]) \right) \times \code{weights}[k]$\\
\codeline Return $\code{indices}$ and $\code{weights}$
\end{pseudocode}
\caption{Computes the multilinear interpolation weights and corresponding vertex indices for a unit lattice cell $[0,1]^D$ and an $x \in [0,1]^D$. Let the lattice parameters be indexed such that $s_d = 2^d$ is the difference in the indices of the parameters corresponding to any two vertices that are adjacent in the $d$th dimension, for example, for the $2 \times 2$ lattice, order the vertices [0 0], [1 0], [0 1], [1 1] and index the corresponding lattice parameters in that order.}
\label{alg:dynamic-programming}
\end{algorithm*}

The following lemma establishes the correctness of Algorithm \ref{alg:dynamic-programming}.

\begin{lemma}[Fast Multilinear Interpolation]  Under its assumptions, Algorithm \ref{alg:dynamic-programming} returns the indices of the $2^D$ parameters corresponding to the vertices of the lattice cell containing $x$:
\begin{equation}
\code{indices}[k] = \sum_{d=0}^{D-1}  \left( \floor{x[d]} + \code{bit}_{i}(k) \right) s_{d}, \textrm{ for } k = 1, 2, \ldots , 2^D\\
\end{equation}
and the corresponding $2^D$ multilinear interpolation weights given by (\ref{eqn:bitwise}).
\end{lemma}

\begin{proof}
At the end of the $D'$th iteration over the dimension in Algorithm \ref{alg:dynamic-programming}:
\begin{align*}
\code{size}\left(\code{indices}\right) &= \code{size}\left(\code{weights}\right) = 2^{D'+1} \\
\code{indices}[k] &= \sum_{d=0}^{D'} \left(\floor{x[d]} +  \code{bit}_d(k) \right) s_d \\
\code{weights}[k] &= \prod_{d=0}^{D'} \left( \left( 1 - \code{bit}_{d}(k) \right) \left( 1 - (x[d] - \floor{x_d}) \right) + \code{bit}_{d}(k) (x[d] - \floor{x_d}) \right).
\end{align*}
The above holds for the $D'=1$ case by
  the initialization and inspection of the loop. It is straightforward to
  verify that if the above hold for $D'$, then they also hold for $D'+1$. Then by induction it holds for $D' = D-1$,  as claimed.
\end{proof}

\subsection{Simplex Linear Interpolation}\label{sec:simplex}
For speed, we propose using a more efficient linear interpolation for lattice regression that linearly interpolates each $x$ from only $D+1$ of the $2^D$ surrounding vertices. Many different linear interpolation strategies have been proposed to interpolate look-up tables using only a subset of the $2^D$ vertices (for a review, see \citet{KangBook}). However, with most strategies it is too computationally expensive to determine exactly \emph{which} of the vertices should be used to interpolate each $x$. The wonder of \emph{simplex interpolation} is that it takes only $O(D \log D)$ operations to determine the $D+1$ vertices needed to interpolate any given $x$, and then only $O(D)$ operations to interpolate the identified $D+1$ vertices.  A comparison of simplex and multilinear interpolation is given in Figure \ref{fig:sandm}  for the same look-up table parameters.

Simplex interpolation was proposed in the color management literature by \citet{Kasson:93}, and independently later by \citet{Rovatti:98}.  Simplex interpolation is also known as the \emph{Lovasz extension} in submodular optimization, where it is used to extend a function defined on the vertices of a unit hypercube to be defined on its interior \citep{bachBook}. 

\begin{figure}[t]
\begin{center}
\begin{tabular}{ccl}
\includegraphics[height=.3\textwidth]{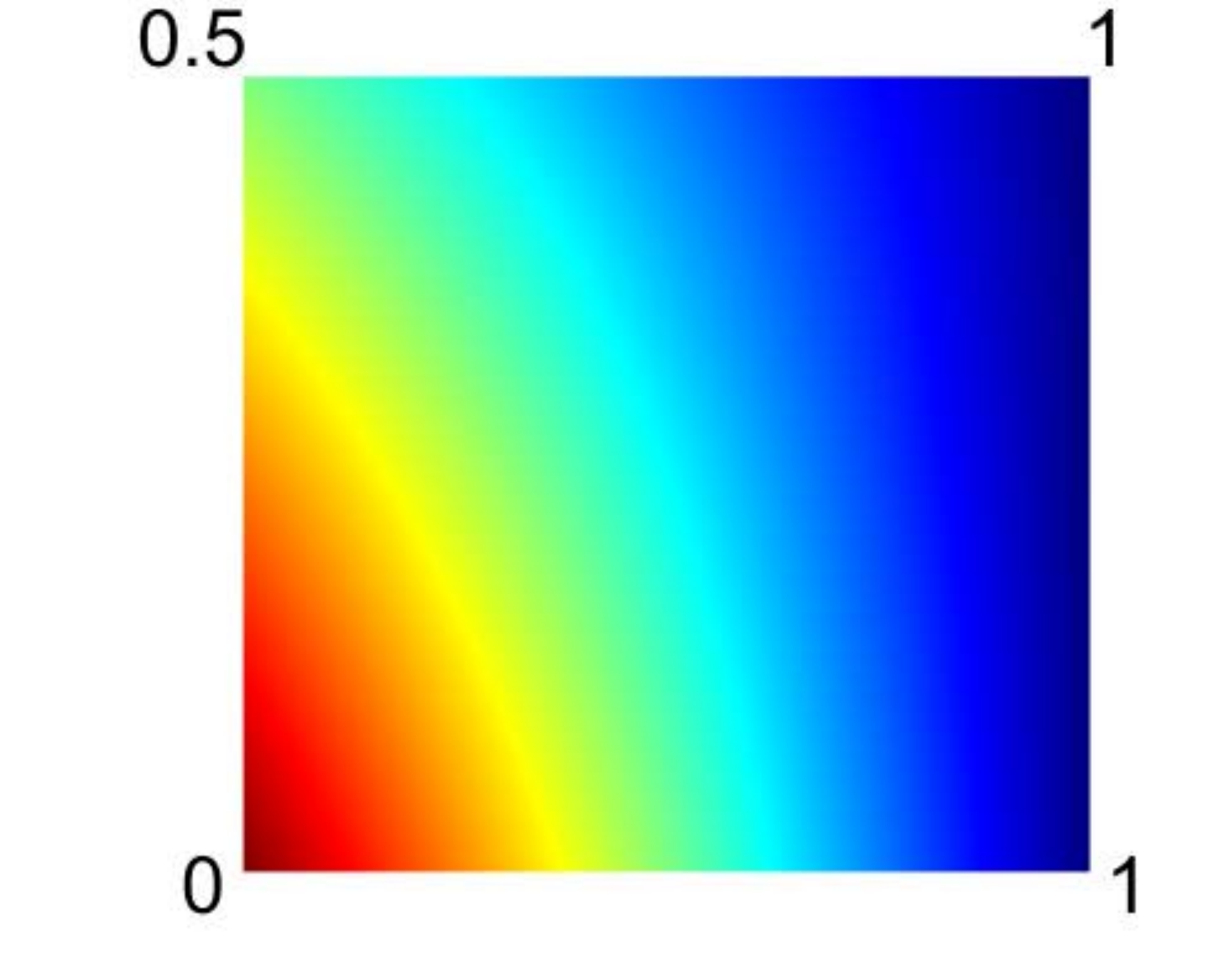} &
\includegraphics[height=.3\textwidth]{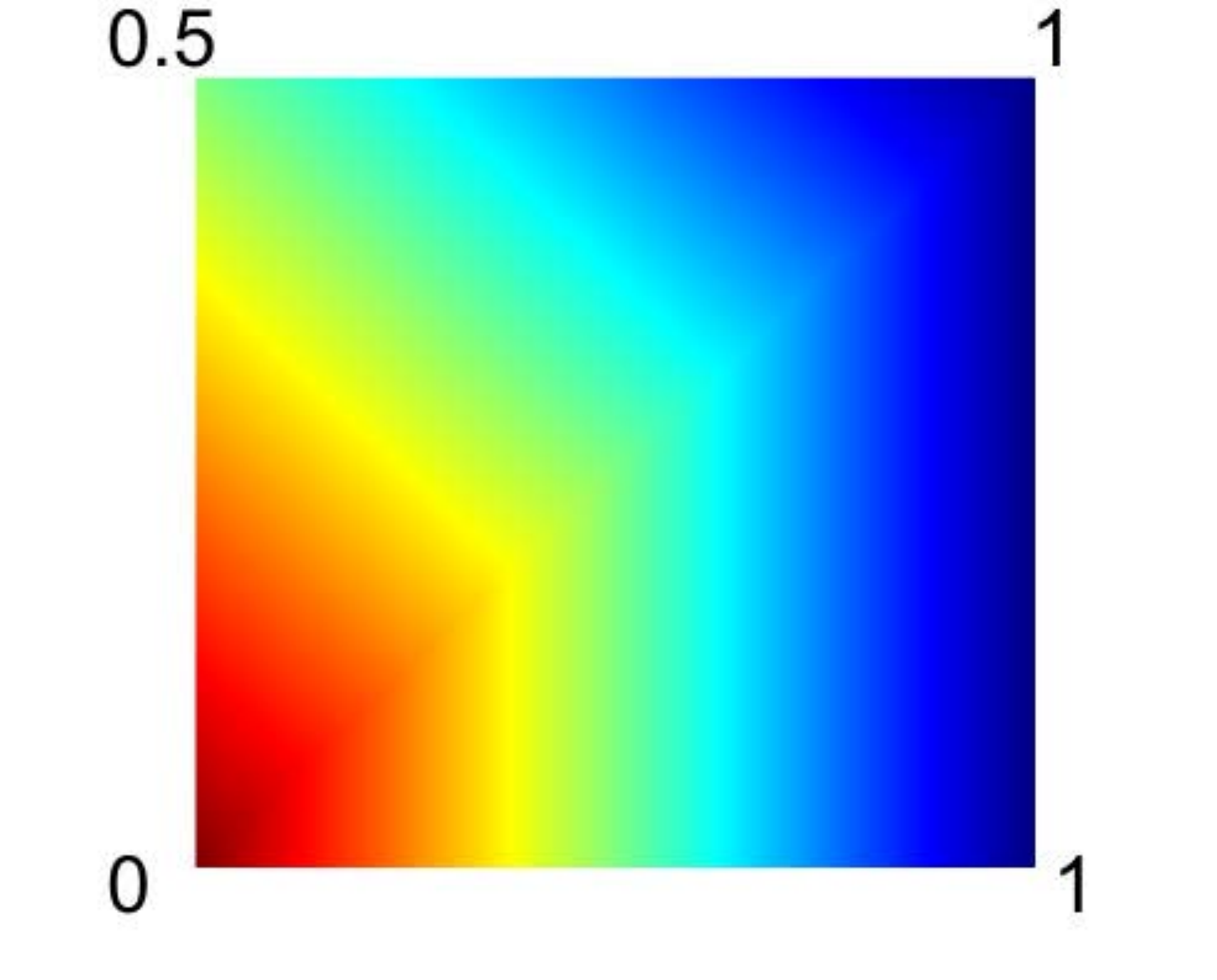} &
\includegraphics[height=.3\textwidth]{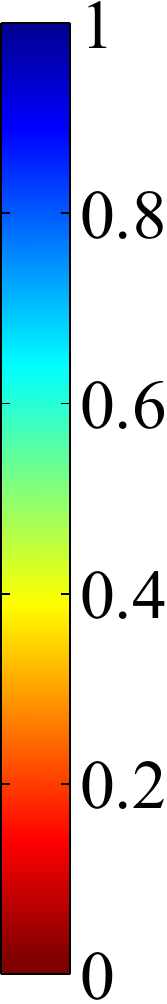} \\
(a) Multilinear Interpolation & (b) Simplex Interpolation &
\end{tabular}
\end{center}
\caption{Both pictures show a $2 \times 2$ look-up table with parameters $0,
0.5, 1$ and $1$. Figure (a) interpolates the look-up table using multilinear
interpolation. Figure (b) interpolates the
look-up table using simplex interpolation, which splits the unit square into two simplices (the upper and lower triangle) and interpolates within each. The function is still continuous because the points along the diagonal are interpolated from only the two corner vertices. Both interpolations produce monotonic functions over both features.}\label{fig:sandm}
\end{figure}

After reviewing how simplex interpolation works, we show in Section \ref{sec:simplexMonotonicity} that it requires the same constraints for monotonicity as multilinear interpolation, and then we discuss how its rotational dependence impacts machine learning in Section \ref{sec:simplexVsMultilinear}. We give example runtime comparisons in Section \ref{sec:runtimes}.

\subsubsection{Partitioning of the Unit Hypercube Into Simplices}\label{sec:simplices}
Simplex interpolation implicitly partitions the hypercube into the set of D! congruent simplices that satisfy the following: each simplex includes the all 0's vertex,  one vertex  is all zeros but has a single 1, one vertex is all zeros but has two 1's, and so on, ending with one vertex that is all 1's, for a total of $D+1$ vertices in each simplex. Figure \ref{fig:simplexFigure} shows the partitioning for the $D=2$  and $D=3$ unit hypercubes.  

This decomposition can also be described by the hyperplanes $x_k = x_r$ for $1 \leq k \leq r \leq D$ \citep{Schmidt:07}. \citet{Knop:73} discussed this decomposition as a special case of Eulerian partitioning of the hypercube, and \citet{Mead:79} showed this is the smallest possible equivolume decomposition of the unit hypercube.

\begin{figure*}[t]
\begin{center}
\begin{tabular}{ m{0.41\textwidth} m{0.41\textwidth} }
\includegraphics[width=.2\textwidth]{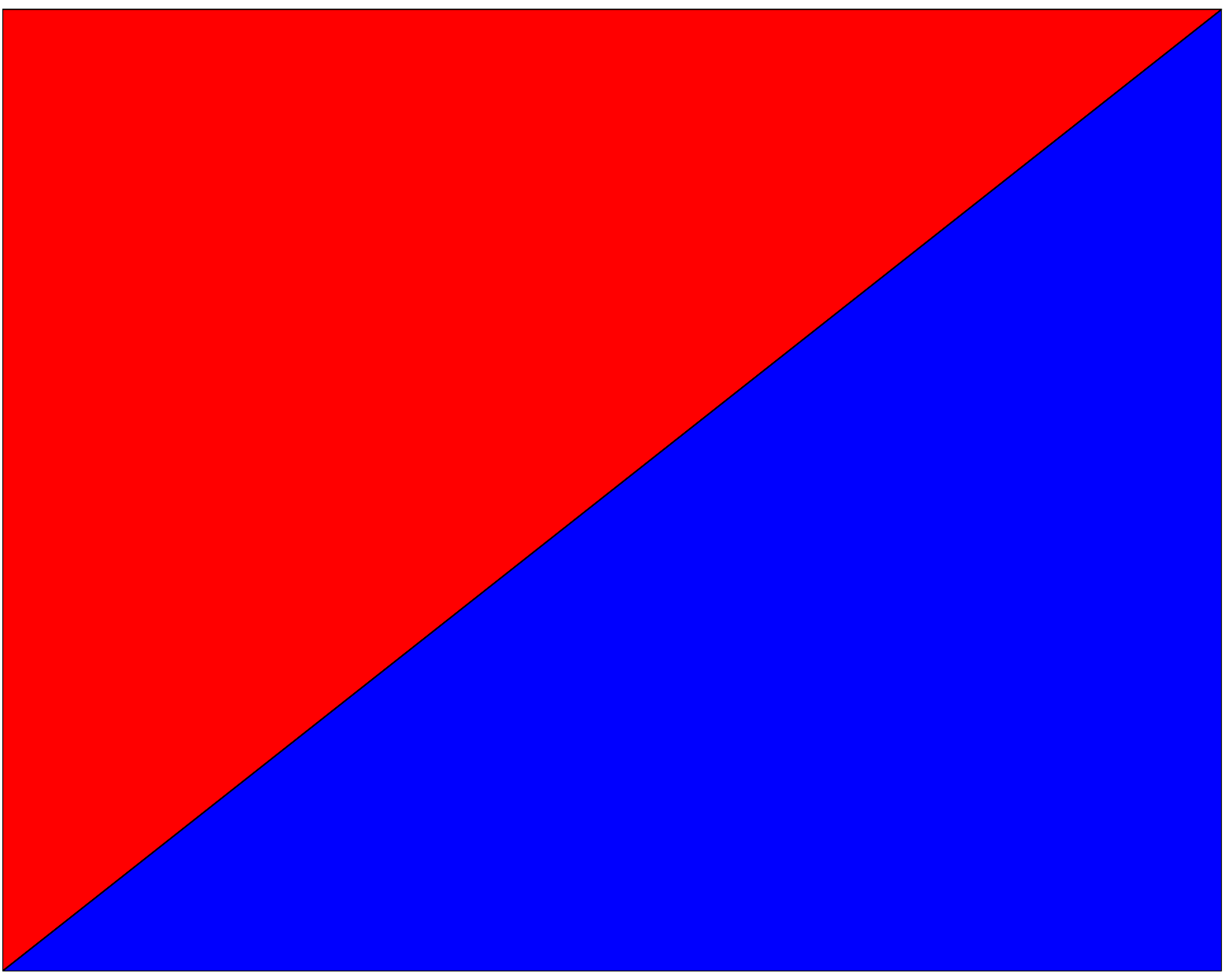} &
\includegraphics[width=.3\textwidth]{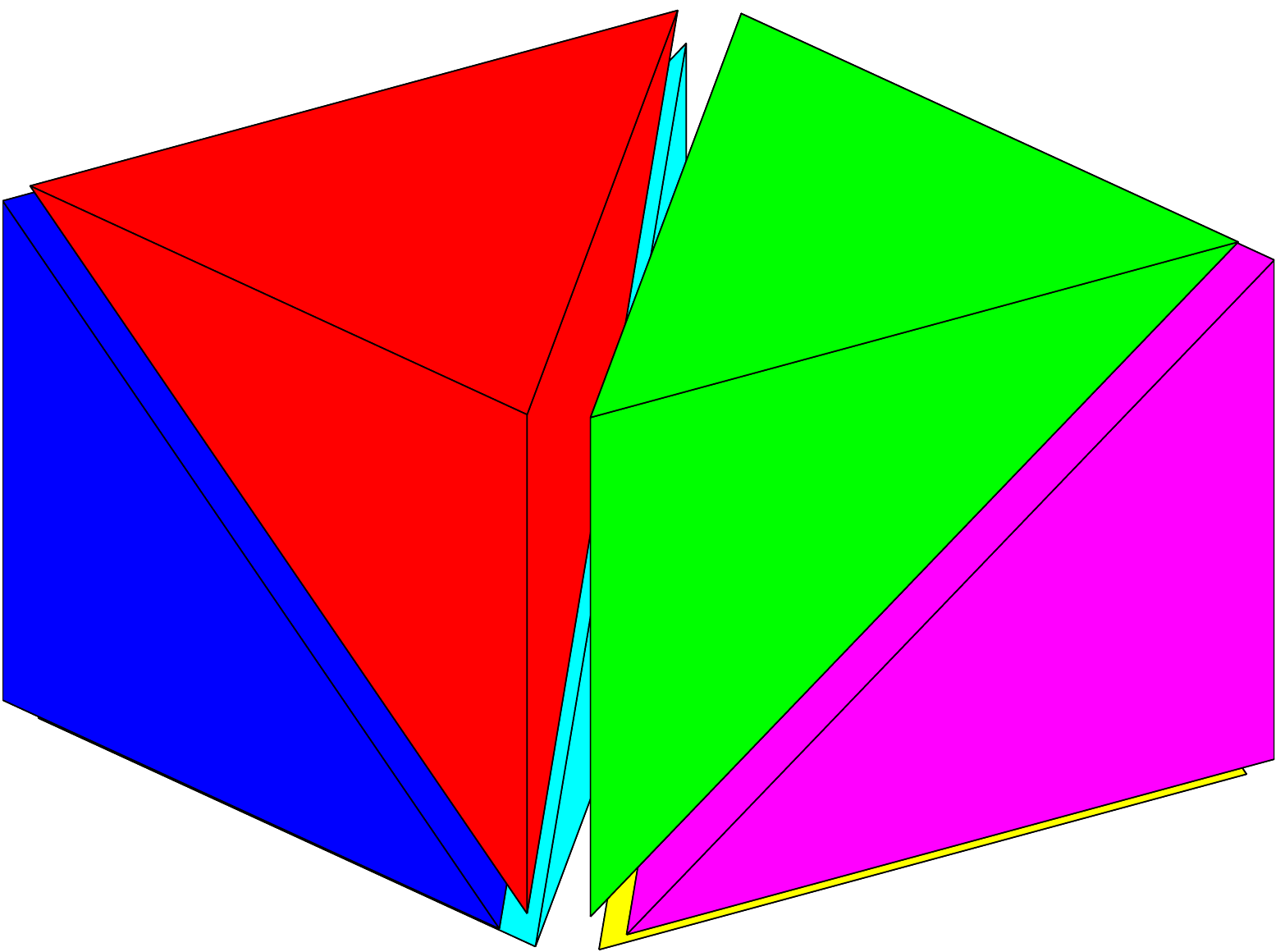}
\end{tabular}
\end{center}
\caption{\textbf{Left:} For the unit square, there are two simplices, one is defined by the three vertices [0 0], [0 1], and [1 1], and the other is defined by the three vertices [0 0], [1 0], and [1 1].  \textbf{Right:}  For the unit cube there are 3! = 6 simplices, each defined by four vertices. The first has vertices: [0 0 0],  [0 0 1],  [0 1 1] , [1 1 1]. The second has vertices: [0 0 0],  [0 0 1],  [1 0 1],  [1 1 1]. And so on. All six simplices have vertices [0 0 0] and [1 1 1], and thus share the diagonal between those two vertices. \\}
\label{fig:simplexFigure}
\end{figure*}

\subsubsection{Simplex Interpolation}\label{sec:simplexInterpolation}
Given $x \in [0,1]^D$, the $D+1$ vertices that specify the simplex that contains $x$ can be computed in $O(D \log D)$ operations by sorting the $D$ values of the feature vector $x$, and then the $d$th simplex vertex has ones in the first $d$ sorted components of $x$. For example, if $x =$[.8  .2  .3],  the $D+1$ vertices of its simplex are [0 0 0], [1 0 0], [1 0 1], [1 1 1]. 

Let $V$ be the $D+1$ by $D$ matrix whose $d$th row is the $d$th vertex of the simplex containing $x$. Then the simplex interpolation weights $\psi(x)$ must satisfy the linear interpolation equations given in (\ref{eqn:linInterp}) such that $\begin{bmatrix} V^T\\ \one^T \end{bmatrix} \psi(x) = \begin{bmatrix} x\\ 1 \end{bmatrix}$. Thus $\psi(x) = \begin{bmatrix} V^T\\ \one^T \end{bmatrix}^{-1} \begin{bmatrix} x\\ 1 \end{bmatrix}$, where because of the highly structured nature of the simplex decomposition the required inverse always exists, and has a simple form such that  $\psi(x)$ is the difference of sequential sorted components of $x$.  For example, for a $2 \times 2$ lattice, and an $x$ such that $x[0] > x[1]$, the simplex interpolation weights $\psi(x)$ are $1 - x[0], x[0] - x[1], x[1]$ on the vertices $[0, 0], [1, 0], [1, 1]$, respectively. The general formula is detailed in Algorithm \ref{alg:simplex}; for more mathematical details see \citet{Rovatti:98}. 



\begin{algorithm*}[ht!]
\begin{pseudocode}
\codename $\code{CalculateSimplexInterpolationWeightsAndParameterIndices}(x)$\\
\codeline Compute the sorted order $\pi$ of the components of $x$ such that $x[\pi[k]]$ is the $k$th largest value of $x$,\\ \codeline \> \> \> that is, $x[\pi[1]]$ is the largest value of $x$, etc.\\
\codeline Initialize $\code{index} = 0$, $\code{indices}[] = [\code{index}]$, $\code{weights}[] = [1]$ \\
\codeline For $d = 1$ to $D$:   \\
\codeline \>  Update $\code{index} = \code{index} + s_{\pi[d]}$ \\
\codeline \>  Append $\code{index}$ to $\code{indices}$ \\
\codeline \>  Update $\code{weights}[d] = \code{weights}[d] - x[\pi[d]]$ \\
\codeline \>  Append $x[\pi[d]]$ to $\code{weights}$\\
\codeline Return $\code{indices}$ and $\code{weights}$
\end{pseudocode}
\caption{Computes the simplex interpolation weights and corresponding vertex indices for a unit lattice cell $[0,1]^D$ and an $x \in [0,1]^D$. Let the lattice parameters be indexed such that $s_d = 2^d$ is the difference in the indices of the parameters corresponding to any two vertices that are adjacent in the $d$th dimension, for example, for the $2 \times 2$ lattice, order the vertices [0 0], [1 0], [0 1], [1 1] and index the corresponding lattice parameters in that order.}
\label{alg:simplex}
\end{algorithm*}

\subsubsection{Simplex Interpolation and Monotonicity}\label{sec:simplexMonotonicity}
We show that the same linear inequality constraints that guarantee monotonicity for multilinear interpolation also guarantee monotonicity with simplex interpolation:

\begin{lemma}[Monotonic Constraints with Simplex Interpolation]
Let $f(x)= \theta^T\phi(x)$ for $\phi(x)$ given in Algorithm \ref{alg:simplex}. The partial derivative $\partial f(x)/\partial x[d] >0$ iff $\theta_{k} > \theta_{k'}$ for all $k,k'$ such that $v_k[d] =0$, $v_{k'}[d] = 1$, and $v_k[m] = v_{k'}[m]$ for all $m \neq d$. 
\end{lemma}

\begin{proof}
Note that Algorithm \ref{alg:simplex} linearly interpolates from $D+1$ vertices at a time, and thus the resulting function is linear over each simplex. Because the parameters are constrained to be increasing, each such linear function is monotonically increasing.  Further, $f(x)$ is continuous for all $x$, because any $x$ on a boundary between simplices only has nonzero interpolation weight on the vertices defining that boundary. In conclusion, the function is piecewise monotonic and continuous, and thus monotonic everywhere.
\end{proof}

\subsubsection{Using Simplex Interpolation for Machine Learning}\label{sec:simplexVsMultilinear}

\begin{figure}[t!]
\begin{center}
\includegraphics[width=.7\textwidth]{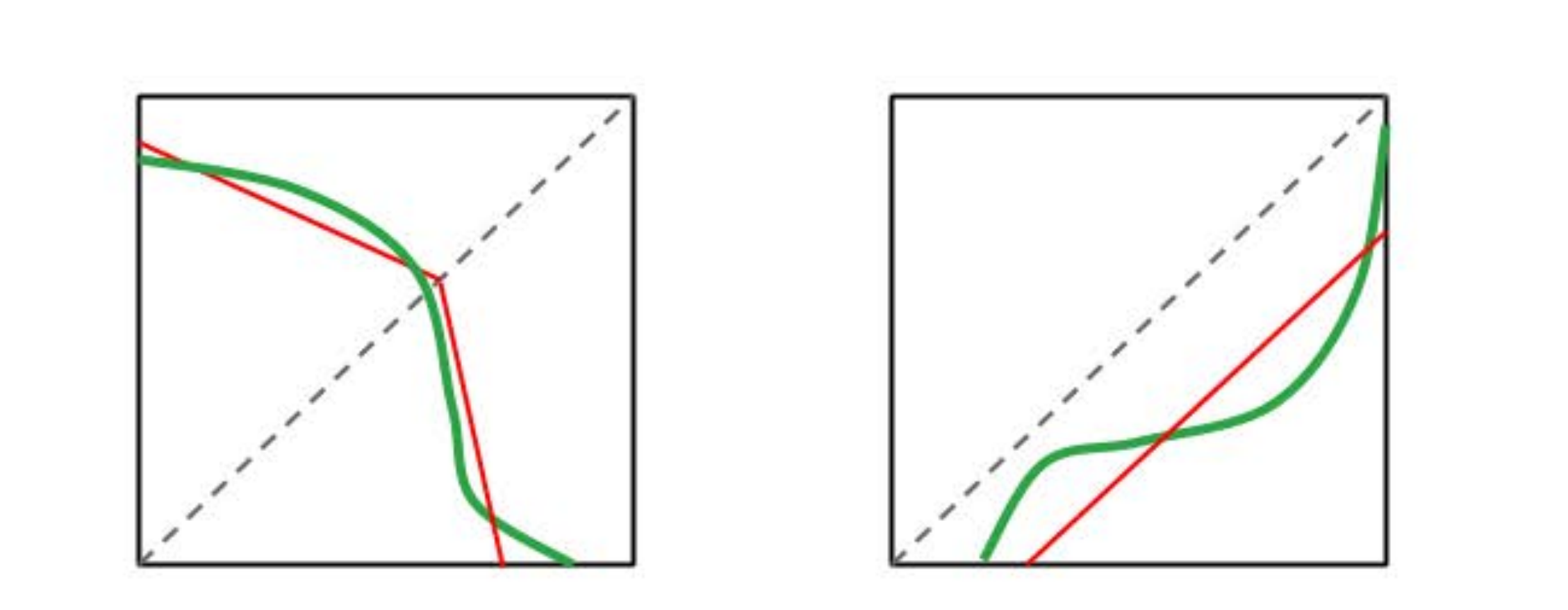}
\end{center}
\caption{Illustrates rotational dependence of simplex interpolation for a $2 \times 2$ lattice and its impact on a binary classification problem. Green thick line denotes the true decision boundary of a binary classification problem. Red thin lines denote the piecewise linear decision boundary fit by lattice regression using simplex interpolation. Dotted gray line separates the two simplices; the function is locally linear over each simplex. \textbf{Left:} The true decision boundary (green) crosses the two simplices. The simplex decision boundary (red) has two linear pieces and can fit the green boundary well. \textbf{Right:}  The same green boundary has been rotated ninety degrees, and now lies entirely in one simplex. The simplex decision boundary (in red) is linear within each simplex, and hence has less flexibility to fit the true green decision boundary.}
\label{fig:simplexBad}
\end{figure}

Simplex interpolation produces a locally linear continuous function made-up of $D!$ hyperplanes defined around the main diagonal axis of the unit hypercube. Compared to multilinear interpolation, simplex interpolation is not as smooth (though continuous), and it is rotationally-dependent. 

For low-dimensional regression problems using a look-up table with many cells, performance of the two interpolation methods has been found to be similar, particularly if one is using a fine-grained lattice with many cells. For example, in a comparison by \citet{SunZhouComparison} for the three-dimensional regression problem of color managing an LCD monitor, multilinear interpolation of a $9 \times 9 \times 9$ look-up table (also called trilinear interpolation in the special case of three-dimensions) produced around $1\%$ worse average error than simplex interpolation, but the maximum error with multilinear interpolation was only $60\%$ of the maximum simplex interpolation error. Another study by \citet{Kang:95} using simulations concluded that the interpolation errors of these methods was ``about the same.''

However, when using a coarser lattice like $2^D$, as we have found useful in practice for machine learning, the rotational dependence of simplex interpolation can cause problems because the flexibility of the interpolated function $f(x)$ differs in different parts of the feature space. Figure \ref{fig:simplexBad} illustrates this for a binary classifier on two features.

To address the rotational dependence, we recommend using prior knowledge to define the features positively or negatively in a way that aligns the simplices' shared diagonal axis along the assumed slope of $f(x)$.  If there are monotonicity constraints, this is done by specifying each feature so that it is monotonically increasing, rather than monotonically decreasing. For binary classification, features should be specified so that the feature vector for the most prototypical example of the negative class is the all-zeros feature vector, and the feature vector for the most prototypical example of a positive class is the all-ones feature vector. This should put the decision boundary as orthogonal to the shared diagonal axis as possible, providing the interpolated function the most flexibility to model that decision boundary.  In addition, for low-dimensional problems, using a finer-grained lattice will produce more flexibility overall, so that the flexibility within each lattice cell is less of an issue. 

Following these guidelines, we surprisingly and consistently find that simplex interpolation of $2^D$ lattices is roughly as accurate as multilinear interpolation, and much faster for $D \geq 8$. This is demonstrated in the case studies of Section \ref{sec:experiments} (runtime comparisons given in Section \ref{sec:runtimes}).

\section{Regularizing the Lattice Regression To Be More Linear}\label{sec:regularizers}
We propose a new regularizer that takes advantage of the lattice structure and encourages the fitted function to be more linear by penalizing differences in parallel edges: 
\begin{equation}\label{eqn:torsion}
R_{\textrm{torsion}}(\theta) = \sum_{d = 1}^D \sum_{\substack{\tilde{d} = 1 \\ \tilde{d} \neq d}}^D \: \: \: \sum_{\substack{r,s,t,u \textrm{ such that}\\ v_r \textrm{ and } v_s \textrm{ adjacent in dimension } d,\\ v_t \textrm{ and } v_u \textrm{ adjacent in dimension } d, \\ v_r \textrm{ and } v_t \textrm{ adjacent in dimension } \tilde{d}}} ((\theta_r - \theta_s)  - (\theta_t - \theta_u))^2.
\end{equation}

This regularizer penalizes how much the lattice function twists from side-to-side, and hence we refer to this as the \emph{torsion} regularizer. The larger the weight on the torsion regularizer in the objective function, the more linear the lattice function will be over each $2^D$ lattice cell.

Figure \ref{fig:regularizerFigure} illustrates the torsion regularizer and compares it to previously proposed lattice regularizers, the standard graph Laplacian \citep{Garcia:09} and graph Hessian  \citep{Garcia:12}.  As shown in the figure, for multi-cell lattices, the torsion and graph Hessian regularizers make the function more linear in different ways, and may both be needed to closely approximate a linear function. Like the graph Laplacian and graph Hessian regularizers, the proposed torsion regularizer is convex but not strictly convex, and can be expressed in quadratic form as $\theta^T K \theta$, where $K$ is a positive semidefinite matrix.

\begin{figure}[h!]
\begin{center}
\includegraphics[width=\textwidth]{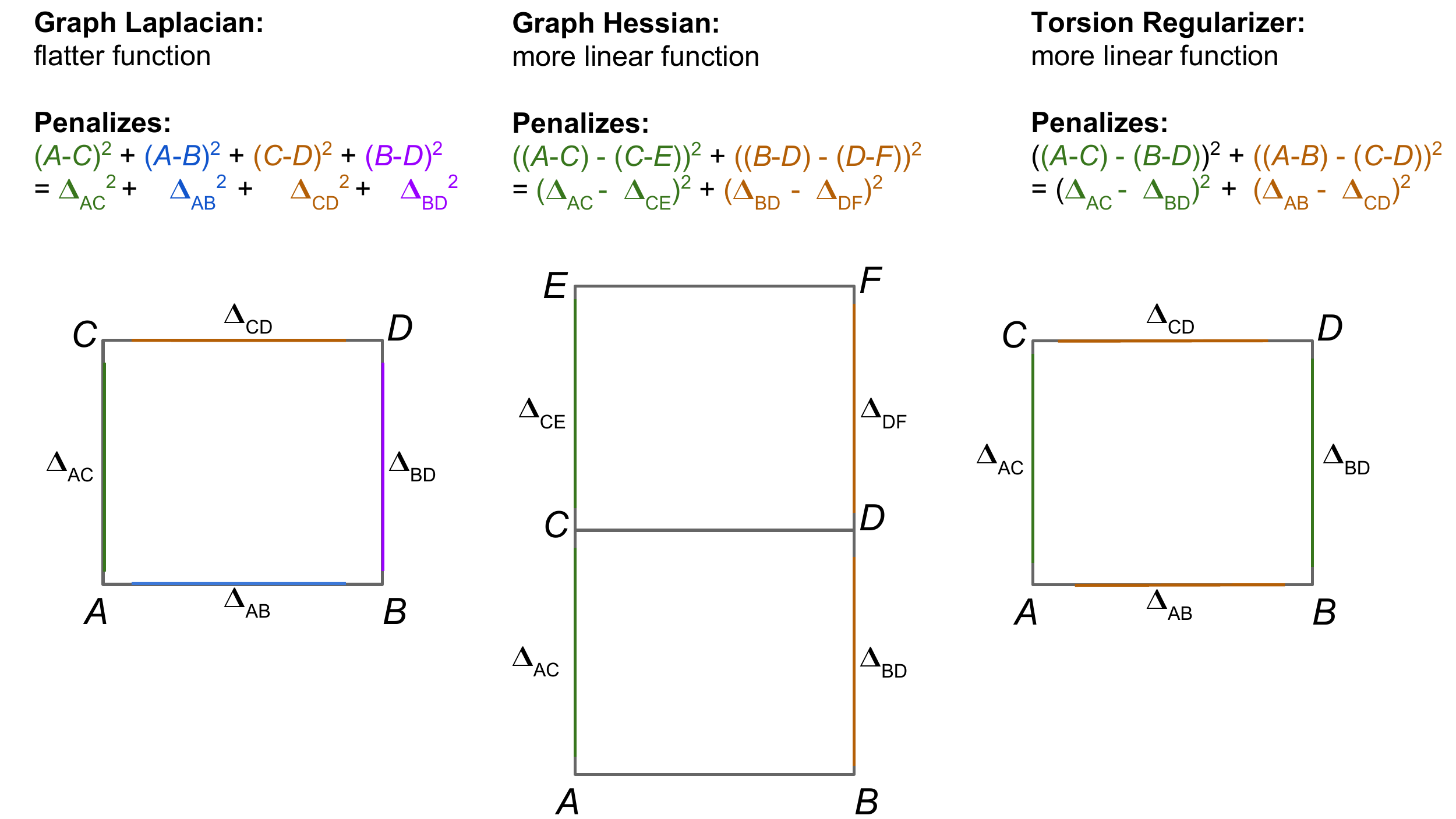}
\end{center}
\caption{Comparison of lattice regularizers. The lattice parameters are denoted by  $A, B, C, D, E$, and $F$.  The deltas indicate the  differences between adjacent parameters along each edge, and thus each delta is the slope of the function along its edge.  Each color corresponds to a different additive term in a regularizer.  The graph Laplacian regularizer \emph{(left)} minimizes the sum of the squared slopes, producing a flatter function.  The graph Hessian regularizer \emph{(middle)} minimizes the change in slope in each direction of a multi-cell lattice, keeping the function from bending too much between lattice cells.  The proposed torsion regularizer \emph{(right)} minimizes the change in slope between sides of the lattice, for each direction, minimizing the twisting of the function.}
\label{fig:regularizerFigure}
\end{figure}

\section{Jointly Learning Feature Calibrations}\label{sec:calibration}

One can learn arbitrary bounded functions with a sufficiently fine-grained lattice, but increasing the number of lattice vertices $M_d$ for the $d$th feature multiplicatively grows the total number of parameters $M = \prod_d M_d$.  However, we find in practice that if the features are first each transformed appropriately, then many problems require only a $2^D$ lattice to capture the feature interactions. For example, a feature that measures distance might be better specified as $\log$ of the distance.  Instead of relying on a user to determine how to best transform each feature, we automate this feature pre-processing by augmenting our function class with $D$ one-dimensional transformations $c_d(x[d])$ that we learn jointly
with the lattice, as shown in Figure \ref{fig:calibrationBlockDiagram}.

\subsection{Calibrating Continuous Features} \label{sec:numericalCalibration}
We calibrate each continuous feature with a one-dimensional monotonic piecewise linear function, as
illustrated in Figure \ref{fig:exampleCalibrations}. Our approach is similar to
the work of \citet{Jebara:07}, which jointly learns monotonic piecewise linear
one-dimensional transformations and a linear function.

This joint estimation makes the objective non-convex, discussed further in
Section \ref{sec:nonconvex}. To simplify estimating the parameters, we treat
the number of changepoints $C_d$ for the $d$th feature as a hyperparameter, and
fix the $C_d$ changepoint locations (also called knots) at equally-spaced quantiles of the feature values.  The changepoint values are then optimized jointly with the lattice parameters,
detailed in Section \ref{sec:nonconvex}.

\begin{figure}[h!]
\begin{center}
\includegraphics[width=\textwidth]{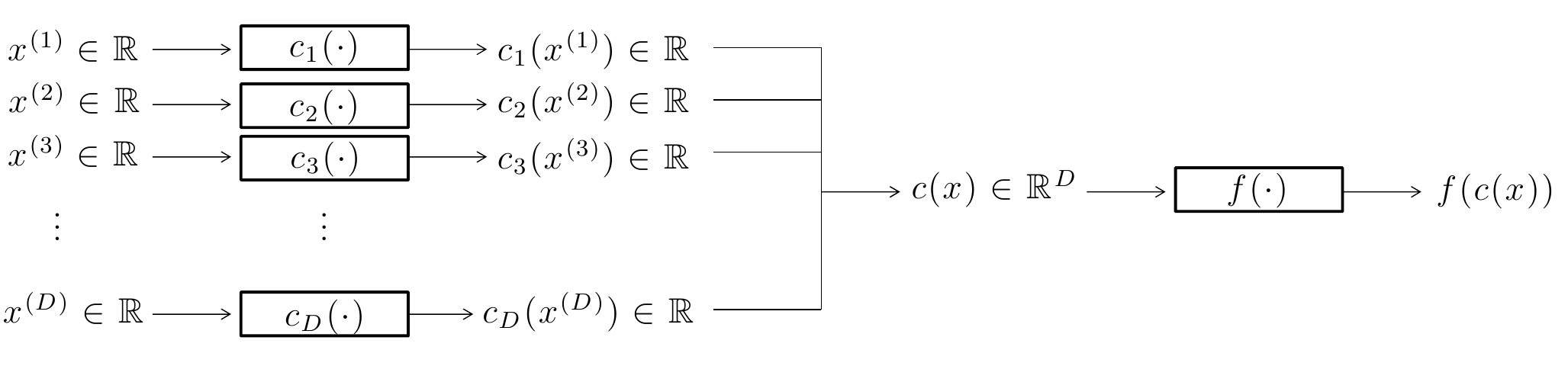}
\end{center}
\caption{Block diagram showing one-dimensional calibration functions
  $\{c_d(\cdot)\}$ applied to each feature before using a lattice $f(\cdot)$ to learn feature-interactions.}
\label{fig:calibrationBlockDiagram}
\end{figure}

\begin{figure}[h]
\begin{center}
\begin{tabular}{cc}
\includegraphics[width=2.0in]{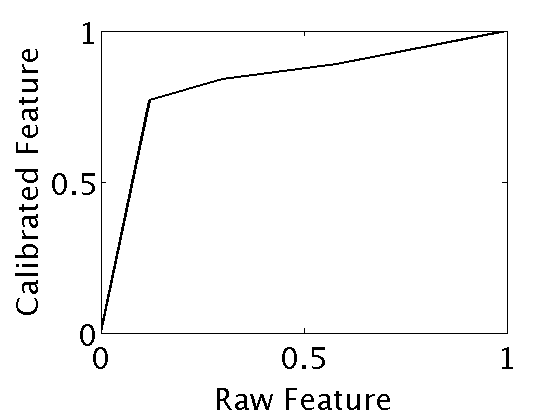} & \includegraphics[width=2.0in]{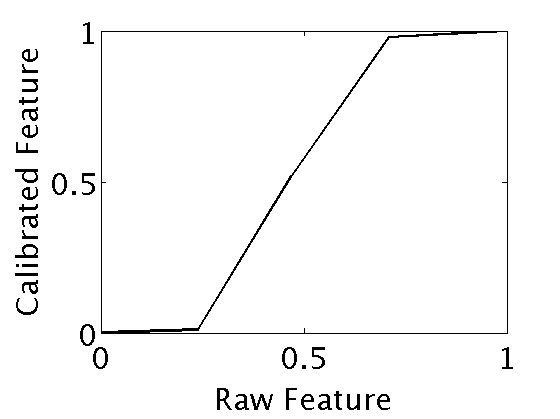} \\
Distance Calibration & Address Similarity Calibration
\end{tabular}
\end{center}
\caption{Learned one-dimensional piecewise linear calibration functions for a distance and
  address-similarity feature for the business-matching case study in Section
  \ref{sec:businessMatching}. \textbf{Left:} The raw distance is measured in
  meters, and its calibration has a log-like effect. \textbf{Right:} The raw address feature is calibrated with a sigmoid-like transformation.}
\label{fig:exampleCalibrations}
\end{figure}

\subsection{Calibrating Categorical Features} \label{sec:categoricalCalibration}
If the $d$th feature is categorical, we propose using a calibration function $c_d(\cdot)$ to map each category to a real value in $[0, M_d - 1]$. That is, let the set of possible categories for the $d$th feature be denoted $\mathcal{G}_d$, then $c_d: \mathcal{G}_d \rightarrow [0,M_d - 1]$. Figure \ref{fig:countryLattice} shows an example lattice with a categorical country feature that has been calibrated to lie on $[0,2]$.  If prior knowledge is given about the ordering of the original discrete values or categories, then partial or full pairwise constraints can be added on the mapped values to respect the known ordering information. These can be expressed as additional sparse linear constraints on pairs of parameters.

\begin{figure}[t]
\begin{center}
\includegraphics[width=4.8in]{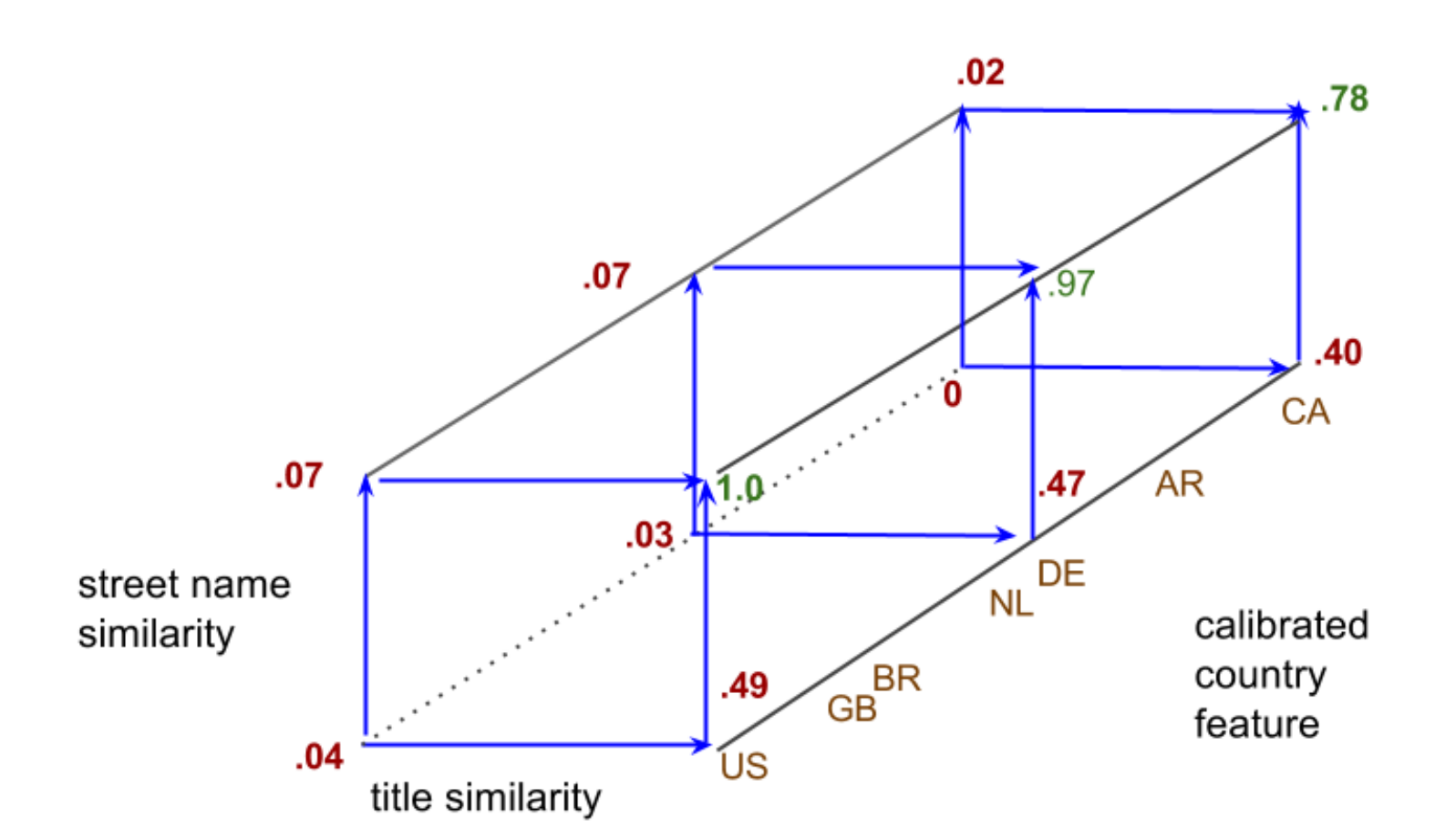}
\end{center}
\caption{
  A $2 \times 2 \times 3$ lattice illustrating calibrating a categorical
  feature. In this example, each sample is a pair of business listings, and
  the problem is to classify whether the two listings are about the same
  business, based on the similarity of their street names,  titles, and the
  country. A  score $f(x)$ is interpolated from the parameters corresponding
  to the vertices of the $2 \times 2 \times 2$ lattice cell in which $x$ lies,
  then thresholded at $0.5$. The red parameter values are below the
  matching threshold of $0.50$, and the green parameters are above the matching
  threshold. The blue arrows denote that the lattice was constrained to be monotonically
  increasing in the street name similarity and the title similarity. In this
  toy example, we only show the calibrated values for a few countries: US maps
  to 0, Great Britain maps to .3, Brazil to .4, Netherlands to .9, Germany to
  1, Argentina to 1.5, and Canada to 2. One can interpret this lattice as
  modeling three classifiers, sliced along the country vertices: a classifier
  for country = 0, one for country = 1, and one for country = 2. Samples from
  Argentina (AR) are interpolated equally from the parameters where country = 1
  and country = 2. Samples from Great Britain, and Netherlands are interpolated
  from the two classifiers specified at country = 0 and 1, with Netherlands
  putting the most weight on the classifier where country = 1.
}
\label{fig:countryLattice}
\end{figure}

\section{Calibrating Missing Data and Using Missing Data Vertices} \label{sec:missing}
We propose two supervised approaches to handle missing values in the training or test set. 

First, one can do a supervised imputation of missing data values by calibrating a missing data value for each feature. This is the same approach proposed for calibrating categorical values in Section \ref{sec:categoricalCalibration}:  learn the numeric value in $[0, M_d - 1]$ to impute if the $d$th feature is missing that minimizes the structural risk minimization obejctive. In this approach, missing data is handled by a calibration function $c_d(\cdot)$, and like the other calibration function parameters. Other researchers have also considered joint training of classifiers and imputations for missing data, for example \citet{vanEsbroeck:2014} and \citet{Carin:2007}.

Second, a more flexible option is to give missing data its own \emph{missing data vertices} in the lattice, as shown in Figure \ref{fig:missingFinal}. This is similar to a decision tree handling a missing data value by splitting a node on whether that feature is missing.  For example, the non-missing feature values can be scaled to $[0, M_d-2]$, and if the data is missing is it mapped to $M_d-1$. This increases the number of parameters but gives the model the flexibility to handle missing data differently than non-missing data.  For example, missing the street number in a business description may correlate with lower quality information for all the features. 

To regularize the lattice parameters corresponding to missing data vertices, we apply the graph regularizers detailed in Section \ref{sec:regularizers}. These could be use to tie any of the parameters to the missing data parameters. In our experiments, for the purposes of graph regularization, we treat the missing data vertices as though they were adjacent to the minimum and maximum vertices of that feature in the lattice.

With either of these two proposed strategies, linear inequalities can be added on the appropriate parameters (the calibrator parameters in the first proposal, or the missing data vertex parameters in the second proposal) to ensure that the function value for missing data is bounded by the minimum and maximum function values, that is, that missing $x[d]$ never produces a smaller $f(x)$ than $x[d] = 0$, nor a larger $f(x)$ than $x[d] = M_d$.

\begin{figure}[t]
\begin{center}
\includegraphics[width=3.5in]{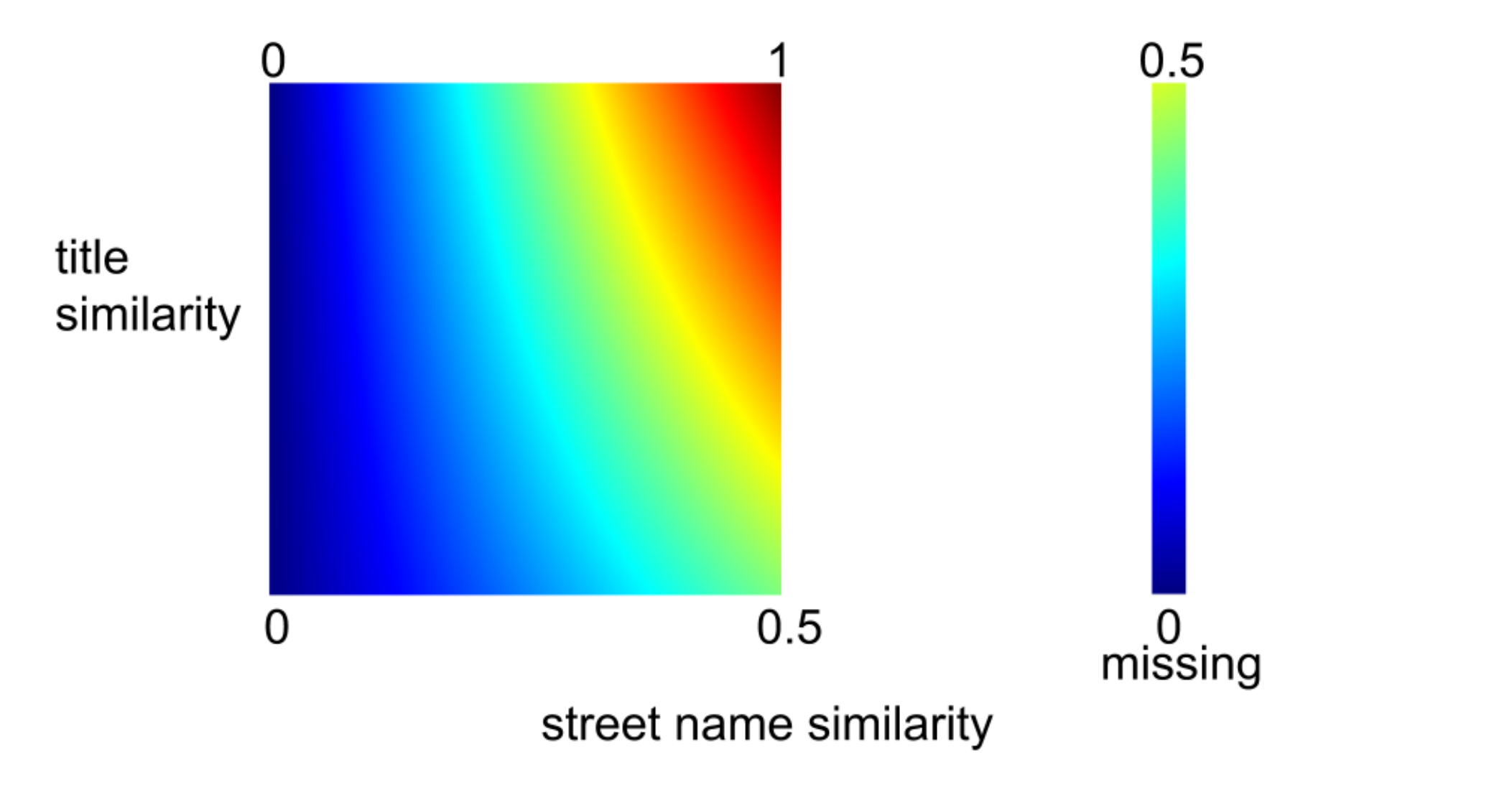}
\end{center}
\caption{Illustration of handling missing data by giving it its own slice of the lattice. In this example, the title similarity and street name similarity values are in $[0,1]^2$, or the street name similarity value is missing. The lattice has $3 \times 2 = 6$ parameters, with the parameter values shown. Given a feature vector where title is $0.5$ and street name similarity is missing, the two parameters corresponding to missing street name similarity would be interpolated to produce the output $f(x) = 0.25$. }
\label{fig:missingFinal}
\end{figure}

\section{Large-Scale Training}\label{sec:bigData}
For convex loss functions $\ell (\theta)$ and convex regularizers $R(\theta)$, any solver for convex problems with linear inequality constraints can be used to optimize the lattice parameters $\theta$ in (\ref{eqn:latticeRegressionMonotonic}).  However, for large $n$ and for even relatively small $D$, training the proposed calibrated monotonic lattices is challenging due to the number of constraints, the number of terms in the graph-regularizers, and the non-convexity created by using calibration functions.

In this section we discuss various standard and new strategies we found useful in practice: our use of stochastic gradient descent (SGD), stochastic handling of the regularizers, parallelizing-and-averaging for distributed training, handling the large number of constraints in the context of SGD, and finally some details on how we optimize the non-convex problem of training the calibrator functions and the lattice parameters.  Throughout this section,  we assume the standard setting of (\ref{eqn:latticeRegressionMonotonic}); the generalization to the pairwise ranking problem of (\ref{eqn:latticeRegressionRanking}) is straightforward. 

\subsection{SGD and Reducing Variance of the Subgradients}
To scale to a large number of samples $n$, we used SGD for all our experiments. For each SGD iteration $t$, a labeled training sample $(x_i, y_i)$
is sampled uniformly from the set of training sample pairs. One finds the corresponding subgradient of (\ref{eqn:latticeRegressionMonotonic}), and takes a tiny step in its negative gradient direction. (The resulting parameters may then violate the constraints, which we discuss in Section \ref{sec:howWeProject}.)

A straightforward SGD implementation for (\ref{eqn:latticeRegressionMonotonic}) would use the subgradient:
\begin{equation}\label{eqn:naiveSGD}
  \Delta = \nabla_{\theta} \ell \left( \theta^T \phi\left( x_i \right), y_i  \right) + \nabla_{\theta} R\left( \theta \right),
\end{equation}
where the $\nabla_{\theta}$ operator finds an arbitrary subgradient of its
argument w.r.t. $\theta$.  Ideally, these subgradients should be cheap-to-compute, so each iteration is fast. The computational cost is dominated by computing the regularizer, if using any of the graph regularizers discussed in Section \ref{sec:regularizers}.  

Because the training example $(x_i,y_i)$ in (\ref{eqn:naiveSGD}) is randomly sampled, the above subgradient is a realization of a stochastic subgradient whose expectation is equal to the true gradient. The number of iterations needed for the SGD to converge depends on the squared Euclidean norms of the stochastic subgradients~\citep{NemirovskiJuLaSh09}, with larger norms resulting
in slower convergence. The expected squared norm of the stochastic subgradient can be decomposed into the sum of two terms: the squared
expected subgradient magnitude, and the variance. We can do little about the expected magnitude, but we can improve the trade-off between the computational cost of each subgradient and the variance of the stochastic subgradients. In the next two sub-sections, we describe two such strategies.

\subsubsection{Mini-Batching} \label{sec:minibatch}
We reduce the variance of the stochastic subgradient's loss term by mini-batching over multiple random samples ~\citep{DekelGiShXi12}.  Let
$\mathcal{S}_{\ell}$ denote a set of $k_{\ell}$ training indices sampled uniformly with replacement from $1, \ldots, n$, then the mini-batched subgradient is:
\begin{equation}
  \Delta = \frac{1}{k_{\ell}} \sum_{i \in \mathcal{S}_{\ell}} \nabla_{\theta}
  \ell\left( \theta^T \phi\left( x_i \right), y_i \right) + \nabla_{\theta}
  R\left( \theta \right).
\end{equation}
This simultaneously reduces the variance and increases the computational cost
of the loss term by a factor of $k_{\ell}$. For sufficiently small $k_{\ell}$, this is a net win because 
differentiating the regularizer is the dominant computational term.

\subsubsection{Stochastic Subgradients for Regularizers}
We propose to reduce the computational cost of each SGD iteration by randomly sampling the additive terms of the regularizer, for regularizers that can be expressed as a sum of terms: $R(\theta) = \sum_{j=1}^m r_j(\theta)$. For example, for a $2^D$ lattice, each calculation of the graph Laplacian regularizer subgradient sums over $m = D 2^{D-1}$ terms, and the graph torsion regularizer subgradient sums over $m = D (D-1) 2^{D-3}$ terms.

 Let $\mathcal{S}_R$ denote a set of $k_R$ indices sampled uniformly with replacement from $1, l\dots,m$, then define the subgradient:
\begin{equation}
 \Delta = \frac{1}{k_{\ell}} \sum_{i \in \mathcal{S}_{\ell}} \nabla_{\theta}
  \ell\left( \theta^T \phi\left( X_i \right), Y_i \right) + \frac{m}{k_R}
  \sum_{j \in \mathcal{S}_R} \nabla_{\theta} r_j\left( \theta \right). \label{eqn:goodSubGradient}
\end{equation}

While this makes the subgradient's regularizer term stochastic, and hence increases the subgradient variance, we find that good choices of $k_{\ell}$ and $k_R$ in (\ref{eqn:goodSubGradient}) can produce a useful tradeoff between the computational cost of computing each subgradient and the number of SGD iterations needed for acceptable converge. For example, in one real-world application using torsion regularization, the choice of $k_R = 1024$ and $k_{\ell} = 1$ led to a 150$\times$ speed-up in training and produced statistically indistinguishable accuracy on a held-out test set. 

\subsection{Parallelizing and Averaging}\label{sec:parallelize}
For a large number of training samples $n$, one can split the $n$ training samples into $K$ sets, then independently and in-parallel train a lattice on each of the $K$ sets. Once trained, the vector lattice parameters for the $K$ lattices can simply be averaged.  This parallelize-and-average approach was investigated for large-scale training of linear models by \citet{MannBigFast}. Their results showed similar accuracies to distributed gradient descent, but $1000\times$ less network traffic and reduced wall-clock time for large datasets. In our implementation of the parallelize-and-average approach we do multiple syncs: averaging the lattices, then sending out the averaged lattice to parallelized workers to keep improving with further training.  We illustrate the performance and speed-up of this simple parallelize-and-average for learning monotonic lattices in Section \ref{sec:youtube} and Section \ref{sec:runtimes}.  A more complicated implementation of this strategy would use the alternating direction method of multipliers with a consensus constraint \citep{BoydADMM:2010}, but that requires an additional regularization towards a local copy of the most recent consensus parameters. 

\subsection{Jointly Optimizing Lattice and Calibration Functions}\label{sec:nonconvex}
To learn a \emph{calibrated} monotonic lattice, we  jointly optimize the calibration functions and the lattice parameters. Let $x$ denote a feature vector with $D$ components, each of which is either a  continuous or categorical value (discrete features can be modeled either as continuous features or categorical as the user sees fit).  Let  $c_d(x[d]; \alpha^{(d)})$ be a calibration function that acts on the $d$th component of $x$ and has parameters $\alpha^{(d)}$.

If the $d$th feature is continuous, we assume it has a bounded domain such that $x[d] \in [l_d, u_d]$ for finite $l_d,u_d \in \mathbb{R}$. Then the $d$th calibration function $c_d(x[d]; \alpha^{(d)})$ is a monotonic piecewise linear transform with fixed knots at $l_d$, $u_d$, and the $C_d-2$ equally-spaced quantiles of $d$th feature over the training set. Let the first and last knots of the piecewise linear function map to the lattice bounds $0$ and $M_d - 1$ respectively (as shown in Figure \ref{fig:exampleCalibrations}), that is, if $C_d = 2$ then $c_d(x[d]; \alpha^{(d)})$ simply linearly scales the raw range $[l_d, u_d]$ to the lattice domain $[0, M_d - 1]$ and there are no parameters $\alpha^{(d)}$.  For $C_d > 2$, the parameters $\alpha^{(d)} \in [0, M_d -1]^{C_d - 2}$ are the $C_d - 2$ output values of the piecewise linear function for the middle $C_d - 2$ knots. 

If the $d$th feature is categorical with finite category set $\mathcal{G}_d$ such that $x[d] \in \mathcal{G}_d$, then the $d$th calibration function maps the categories to the lattice span such that $c_d(x[d]; \alpha^{(d)}): \mathcal{G}_d \rightarrow [0, M_d-1]$ and the parameters are the $| \mathcal{G}_d | $ categorical mappings such that  $c_d(x[d]; \alpha^{(d)}) = \alpha^{(d)}[k]$ if $x[d]$ belongs to category $k$ and $\alpha^{(d)} \in [0, M_d -1]^{|\mathcal{G}_d |}$.

Let $c(x; \alpha)$ denote the vector function with $d$th component function $c_d(x[d]; \alpha^{(d)})$, and note $c(x; \alpha)$ maps a feature vector $x$ to the domain $\mathcal{M}$ of the lattice function. Use $e_d$ to denote the standard unit basis vector that is one for the $d$th component and zero elsewhere with length $D$, then one can write:
\begin{equation}\label{eqn:calibrationFull}
c(x; \alpha) = \sum_{d=1}^D e_d c_d(e_d^T x; \alpha^{(d)}),
\end{equation}

Then the proposed calibrated monotonic lattice regression objective expands the monotonic lattice regression objective  (\ref{eqn:latticeRegressionMonotonic}) to: 

\begin{equation*}\label{eqn:calibratedLatticeRegression} 
\arg \min_{\theta,\alpha} \sum_{i=1}^n \ell (y_i, \theta^T \phi(c(x_i, \alpha)) + R(\theta) \textrm{ s.t. }  A \theta \leq b \textrm{ and } \tilde{A} \alpha \leq \tilde{b}, 
\end{equation*}
where each row of $A$ specifies a monotonicity constraint for a pair of adjacent lattice parameters (as before), and each row of $\tilde{A}$ similarly specifies a monotonicity constraint for a pair of adjacent calibration parameters for one of the piecewise linear calibration functions.

This turns the convex optimization problem (\ref{eqn:latticeRegressionMonotonic}) into a non-convex problem that is
marginally convex in the lattice parameters $\theta$ for fixed $\alpha$, but not necessarily convex with respect to $\alpha$ even if $\theta$ is fixed. Despite the non-convexity of the objective, in our experiments we found sensible and effective solutions by using projected SGD, updating $\theta$ and $\alpha$ with the appropriate stochastic subgradient for each $x_i$. Calculate the subgradient w.r.t. $\theta$ holding $\alpha$ constant, essentially the same as before. Calculate the subgradient w.r.t  $\alpha$ by holding $\theta$ constant and using the chain rule:
\begin{equation}
\frac{\partial \theta^T \phi(c(x_i, \alpha))}{\partial \alpha^{(d)}} = \frac{\partial \theta^T \phi(c(x_i, \alpha))}{\partial c(x_i, \alpha)} \frac{\partial c(x_i, \alpha)}{\partial \alpha^{(d)}}.
\end{equation}

If the $d$th feature is categorical, the partial derivative is $1$ for the calibration mapping parameter corresponding to the category of $x_i[d]$ and zero otherwise: 
\begin{equation}
\frac{\partial c(x_i, \alpha) }{\partial \alpha^{(d)}[k]} = 1 \textrm{ if } x_i[d] \textrm{ is the $k$th category and $0$ otherwise}. 
\end{equation}

If the $d$th feature is continuous,  then the parameters $\alpha^{(j)}[d]$ are the
values of the calibration function at the knots of the piecewise linear function. If $x_i[d]$ lies 
between the $k$th and $(k+1)$th knots at (fixed) positions $\beta_k$ and $\beta_{k+1}$, then
\begin{align*}
\frac{\partial c(x_i, \alpha) }{\partial \alpha^{(d)}[k]} = \frac{(\beta_{k+1} - x_i[d])}{(\beta_{k+1} - \beta_k)} \\
\frac{\partial c(x_i, \alpha) }{\partial \alpha^{(d)}[k + 1]} = \frac{(x_i[d] - \beta_k)}{(\beta_{k+1} - \beta_k)}, 
\end{align*}
and the partial derviative is zero for all other components of $\alpha^{(d)}$. After taking an SGD step that updates $\alpha^{(d)}[k]$ and $\alpha^{(d)}[k + 1]$, the $\alpha^{(d)}$ may violate the monotonicity constraints that ensure a monotonic calibration function, which can be fixed with a projection onto the constraints (see Section \ref{sec:howWeProject} for details).

A standard strategy with nonconvex gradient descent is to try multiple random initializations of the parameters. We did not explore this avenue; instead we simply try to initialize sensibly. Each lattice parameter is initialized to be the sum of its monotonically increasing components (multiply by -1 for any monotonically decreasing components) so that the lattice initialization respects the monotonicity constraints and is a linear function. The piecewise linear calibration functions are initialized to scale linearly to $[0, M_d-1]$. The categorical calibration parameters are ordered by their mean label, then spaced uniformly on $[0, M_d-1]$ in that order.

\subsection{Large-Scale Projection Handling}\label{sec:howWeProject}
Standard projected stochastic gradient descent projects the parameters onto the constraints after each stochastic gradient update. Given the extremely large number of linear inequality constraints needed to enforce monotonicity for even small $D$, we found a full projection each iteration impractical and un-necessary.  We avoid the full projection each iterate by using one of two strategies.

\subsubsection{Suboptimal Projections}
We found that modifying the SGD update to approximate the projection worked well.   Specifically, for each new stochastic subgradient $\eta\Delta$, we create a set of active constraints initialized to $\emptyset$, and, starting from the last parameter values, move along the portion of $\eta\Delta$ that is orthogonal to the current active set until we encounter a constraint, add this constraint to the active set, and then continue until the update $\eta\Delta$ is exhausted or it is not possible to move orthogonal to the current active set. At all times, the parameters satisfy the constraints.  It can be particularly fast because it is possible to exploit the sparsity of the monotonicity constraints (each of which depends on only two parameters) and the sparsity of $\Delta$ (when using simplex interpolation) to optimize the implementation.

But, this strategy is sub-optimal because we do not remove any constraints from the active set during each iteration, and thus parameters can ``get stuck'' at  a corner of the feasible set, as illustrated in Figure \ref{fig:jiggle}. In practice, we found such problems resolve themselves because the stochasticity of the subsequent stochastic gradients eventually jiggles the parameters free. Experimentally, we found this suboptimal strategy to be very effective and to produce statistically similar objective function values and test accuracies more optimal approaches.  All of the experimental results reported in this paper used this strategy. See Section \ref{sec:runtimes} for example runtimes.

\begin{figure}[t!]
\begin{center}
\begin{tabular}{cc}
\includegraphics[width=2.0in]{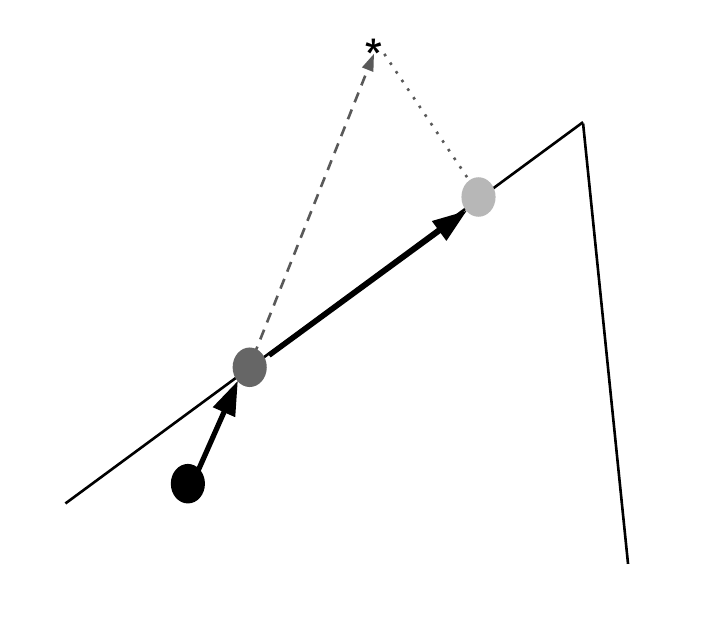} &  \includegraphics[width=2.0in]{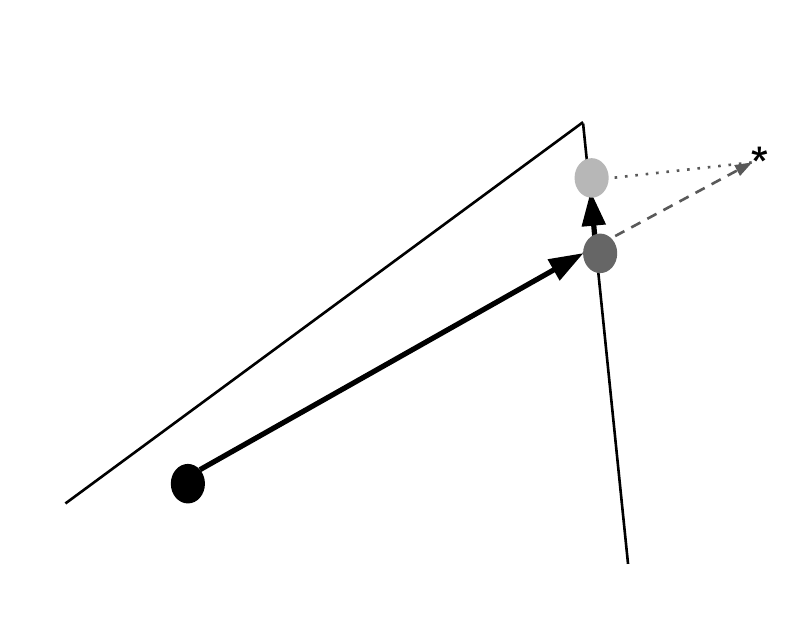} \\
(a) & (b) \\
\includegraphics[width=3.0in]{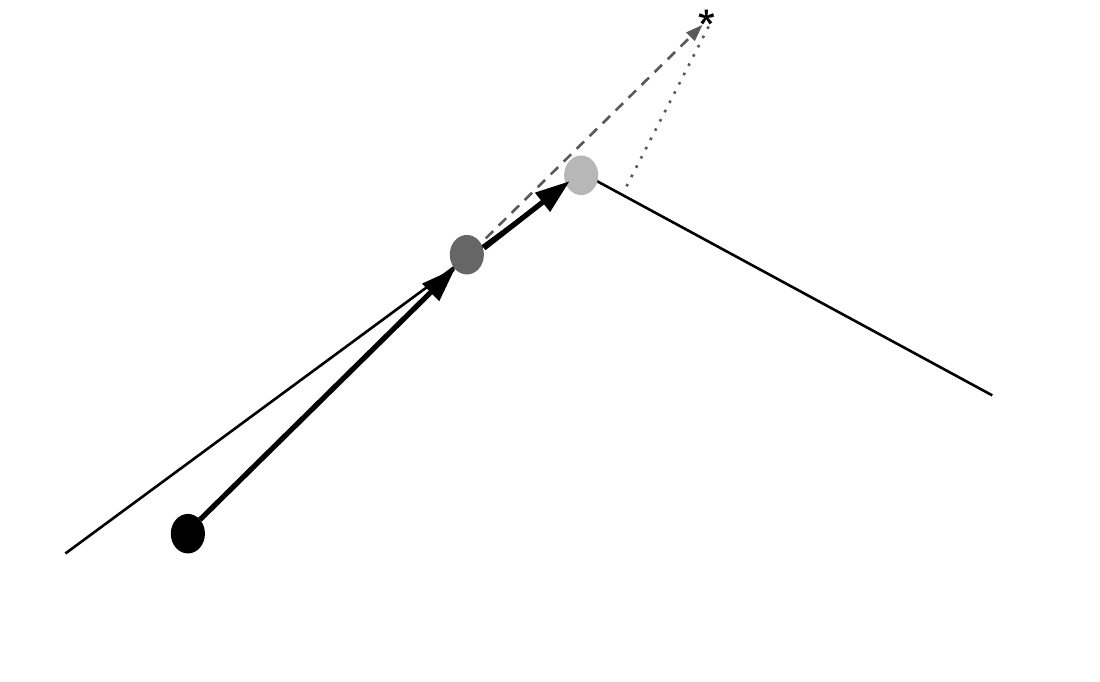} &  \includegraphics[width=3.0in]{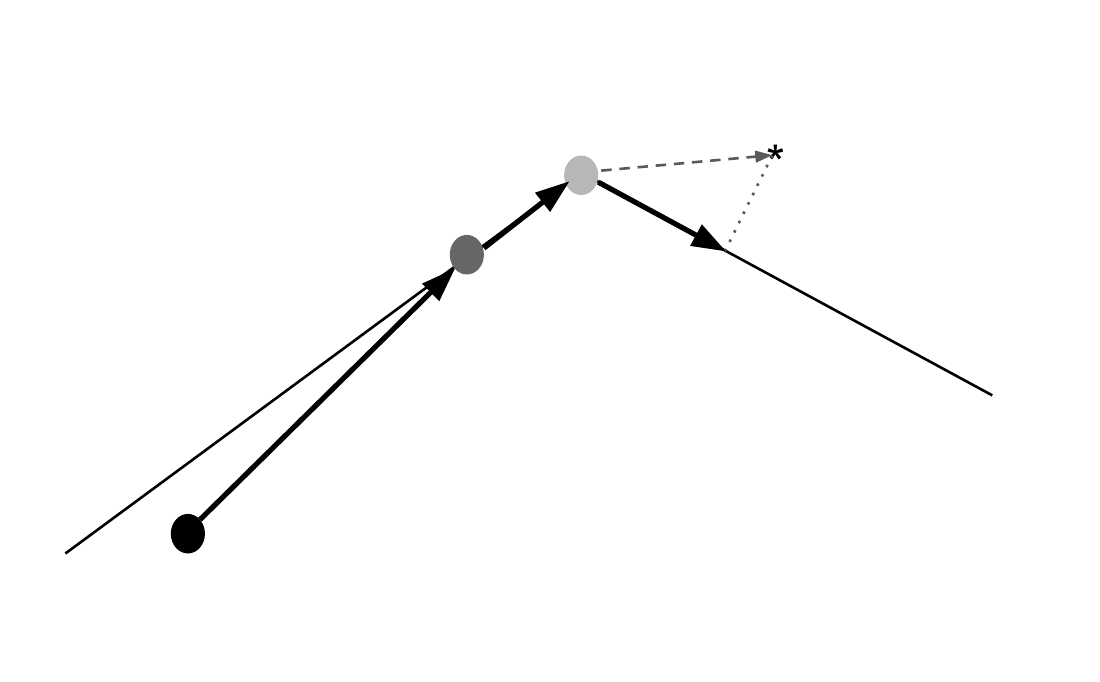} \\
(c) & (d) \\
\end{tabular}
\end{center}
\caption{Four examples of the suboptimal projection stochastic gradient descent step described in Section \ref{sec:howWeProject}. In each case, the constraints are marked by thin solid lines, the black dot represents the parameters at the end of the last SGD iteration, and the new full stochastic gradient descent update is marked by the dashed line, ending in a star.  The optimal projection of the star onto the constraints is marked by the dotted line.   The stochastic gradient is followed until it hits a constraint, and then the component of the remaining gradient orthogonal to the active constraint is applied. The update ends at the light gray dot. In cases (a) and (b), the resulting light dot is the optimal projection of the star onto the constraints.  But in case (c), first one constraint is hit, and then another constraint is hit, and the update gets stuck at the corner of the feasible set without being able to apply all of the stochastic gradient. The resulting light gray dot is \emph{not} the projection of the star onto the constraints, hence the projection for this iteration is suboptimal. However, it is likely that a future stochastic gradient will jiggle the optimization loose, as pictured in (d), producing an update that is again an optimal projection of the latest stochastic gradient.}
\label{fig:jiggle}
\end{figure}

\subsubsection{Stochastic Constraints with LightTouch}
An optimal approach we compared with for handling large-scale constraints is called \emph{LightTouch} \citep{LightTouchArxiv2015}. At each iteration, LightTouch does not project onto any constraints, but rather moves the constraints into the objective, and applies a random subset of constraints each iteration as stochastic gradient updates to the parameters, where the distribution over the constraints is learned as the optimization proceeds to focus on constraints that are more likely to be active. This replaces the per-iteration projections with cheap gradient updates. Intermediate solutions may not satisfy all the constraints, but one full projection is performed at the very end to ensure final satisfaction of the constraints. Experimentally, we found LightTouch to generally lead to faster convergence (see \citet{LightTouchArxiv2015} for its theoretical convergence rate), while producing similar experimental results to the above approximate projected SGD.  LightTouch does require a more complicated implementation to effectively learn the distribution over the constraints. 

\subsubsection{Adapting Stepsizes with Adagrad}
One can generally improve the speed of SGD with adagrad~\citep{Duchi:2011}, even for nonconvex problems \citep{Gupta:2014}. Adagrad decays the step-size adaptively for each parameter, so that parameters updated more often or with larger magnitude gradients have a smaller step size.  We found adagrad did speed up convergence slightly, but required more complicated implementation to correctly handle the constraints, as the projections must be with respect to the adagrad norm rather than the Euclidean norm. We experimented with approximating the adagrad norm projection with the Euclidean projection, but found this approximation resulted in poor convergence. 

\section{Case Studies}\label{sec:experiments}
We present a series of experimental case studies on real world problems to demonstrate different aspects of the proposed methods, followed by some example runtimes for interpolation and training in Section \ref{sec:runtimes}.

Previous datasets used to evaluate monotonic algorithms have been small, both in the number of samples and the number of dimensions, as detailed in Table \ref{tab:relatedWork}. In order to produce statistically significant experimental results, and to better demonstrate the practical need for monotonicity constraints, we use real-world case studies with relatively large datasets, and for which the application engineers have confirmed that they expect or want the learned function to be monotonic with respect to some subset of features.  The datasets used are detailed in Table \ref{tab:datasetSummary}, and include datasets with eight thousand to 400 million samples, and nine to sixteen features, most of which are constrained to be monotonic.

The case studies demonstrate that for problems where the monotonicity assumption is warranted, the proposed calibrated monotonic lattice regression produces similar accuracy to random forests. Random forests is an unconstrained method that consistently provides competitive results on benchmark datasets, compared to many other types of machine learning methods \citep{rfsRock:14}).  

Because any bounded function can be expressed using a sufficiently fine-grained interpolation look-up table, we expect that with appropriate use of regularizers, monotonic lattice regression will perform similarly to other guaranteed monotonic methods that use a flexible function class and are appropriately regularized, such as monotonic neural nets (see \ref{sec:monotonicNeuralNets}). However, of guaranteed monotonic methods, the only monotonic strategy that has been demonstrated to scale to the number of training samples and the number of features treated in our case studies is linear regression with non-negative coefficients (see Table \ref{tab:relatedWork}). 




\subsection{General Experimental Details}
We used ten-fold cross-validation on each training set to choose hyperparameters, including: whether to use graph Laplacian regularization or torsion regularization, how much regularization (in powers of ten), whether to calibrate missing data or use a missing data vertex, the number of change-points if feature calibration was used from the choices: $\{2, 3, 5, 10, 20, 50\}$, and the number of vertices for each feature was started at $2$ and increased by 1 as long as cross-validation accuracy increased.  The step size was tuned using ten-fold cross-validation and choices were powers of 10; it was usually chosen to be one of $\{.01, .1, 1\}$.  If calibration functions were used, a hyperparameter was used to scale the step size for the calibration function gradients compared to the lattice function gradients; this calibration step size scale was also chosen using ten-fold cross-validation and powers of 10, and was usually chosen to be one of $\{.01, .1, 1, 10\}$. Multilinear interpolation was used unless it is noted that simplex interpolation was used. The loss function was squared error, unless noted that logistic loss was used.

Comparisons were made to random forests \citep{Breiman:01}, and to linear models, with either the logistic loss (logistic regression) or squared error loss (linear regression), and a ridge regularizer on the linear coefficients, with any categorical or missing features converted to Boolean features. All comparisons were trained on the same training set, hyperparameters were tuned using cross-validation, and tested on the same test set.  Statistical significance was measured using a binomial statistical significance test with a p-value of .05 on the test samples rated differently by two models.

\begin{table}
\begin{center}
\begin{tabular}{l|rrrrr}
& \# Training  & \# Test  &  & \# Lattice  \\ 
Dataset & Samples & Samples & \# Features  &Parameters \\ 
\hline
Business Matching & 8,000 & 4,000 & 9 & 1728 \\
Ad--Query Matching & 235,996 & 58,224 & 5  &  32 \\
Rendering Classifier & 20,000 & 2,500 & 16 & 65,536 \\
Fusing Pipelines & 1.6 million & 390k & 12 & 24,576 \\
Video Ranking & 400 million & 25 million & 12 & 531,441\\
\end{tabular}
\end{center}
\caption{Summary of datasets used in the case studies.}\label{tab:datasetSummary}
\end{table}

\subsection{Case Study: Business Entity Resolution} \label{sec:businessMatching}
In this case study, we compare the relative impact of several of our proposed extensions to lattice regression.  The business entity resolution problem is to determine if two business descriptions refer to the same real-world business. This problem is also treated by \citet{Dalvi:2014}, where they focus on defining a good title similarity. Here, we consider only the problem of fusing different  similarities (such as a title similarity and phone similarity) into one score that predicts whether a pair of businesses are the same business.  The learned function is required to be monotonically increasing in seven attribute similarities, such as the similarity between the two business titles and the similarity between the street names. There are two other features with no monotonicity constraints, such as the geographic region, which takes on one of 14 categorical values. Each sample is derived from a pair of business descriptions, and a label provided by an expert human rater indicating whether that pair of business descriptions describe the same real-world business. We measure accuracy in terms of whether a predicted label matches the ground truth label, but in actual usage, the learned function is also used to rank multiple matches that pass the decision threshold, and thus a strictly monotonic function is preferred to a piecewise constant function.  The training and test sets, detailed in Table \ref{tab:datasetSummary}, were randomly split from the complete labeled set. Most of the samples were drawn using active sampling, so most of the samples are difficult to classify correctly.

Table \ref{tab:businessMatching} reports results. The linear model performed poorly, because there are many important high-order interactions between the features.  For example, the pair of businesses might describe two pizza places at the same location,  one of which recently closed, and the other recently opened.  In this case, location-based features will be strongly positive, but the classifier must be sensitive to low title similarity to determine the businesses are different. On the other hand, high title similarity is not sufficient to classify the pair as the same, for example, two Starbucks cafes across the street from each other in downtown London. 

The lattice regression model was first optimized using cross-validation, and then we made the series of minor changes (with all else held constant) listed in Table \ref{tab:businessMatching} to illustrate the impact of these changes on accuracy. First, removing the monotonicity constraints resulted in a statistically significant drop in accuracy of half a percent.  Thus it appears the monotonicity constraints are successfully regularizing given the small amount of training data and the known high Bayes error in some parts of the feature space. Lattice regression without the monotonicity constraints performed similarly to random forests (and not statistically significantly better), as expected due to the similar modeling abilities of the methods. 

The cross-validated lattice was $3 \times 3 \times 3 \times 2^6$, where the first three features used a missing data vertex (so the non-missing data is interpolated from a $2^9$ lattice).  Calibrating the missing values for those three features instead of using missing data vertices statistically significantly dropped the accuracy from $81.9\%$ to $80.7\%$. (However, if one subsamples the training set down to 3000 samples, then the less flexible option of calibrating the missing values works better than using missing data vertices.) 

The cross-validated calibration used five changepoints for two of the four continuous features, and no calibration for the two other continuous features. Figure \ref{fig:exampleCalibrations} shows the calibrations learned in the optimized lattice regression.   Removing the continuous signal calibration resulted in a statistically significant drop in accuracy. 

Another important proposal of this paper is calibrating categorical features to real-valued features. For this problem, this is applied to a feature specifying which of 14 possible geographical categories the businesses are in.  Removing this geographic feature statistically significantly reduced the accuracy by half a percent.

The amount of torsion regularization was cross-validated to be $10^{-4}$. Changing to graph Laplacian and re-optimizing the amount of regularization decreased accuracy slightly, but not statistically significantly so. This is consistent with what we often find:  torsion is often slightly better, but often not statistically significantly so, than the graph Laplacian regularizer.

Changing the multilinear interpolation to simplex interpolation (see Section \ref{sec:simplex}) dropped the accuracy slightly, but not statistically significantly.  For some problems we even see simplex interpolation provide slightly better results, but generally the accuracy difference between simplex and multilinear interpolation is negligible. 

\begin{table}
\begin{center}
\begin{tabular}{lrc}
& Test Set Accuracy & Monotonic Guarantee? \\ \hline
Linear Model & 66.6\%  & yes \\
Random Forest & 81.2\% & no \\
Lattice Regression, Optimized  & 81.9\% & yes \\
... Remove Monotonicity Constraints & 81.4\% &  no \\
... Calibrate All Missing Data   & 80.7\% & yes \\
... Remove Calibration & 81.1\% & yes\\
... Remove the Geographic Feature & 81.4\% & yes\\
... Change to Graph Laplacian & 81.7\% & yes \\
... Change to Simplex Interpolation & 81.6\% &  yes \\
\end{tabular}
\end{center}
\caption{Comparison on a business entity resolution problem.} \label{tab:businessMatching}
\end{table}

\subsection{Case Study: Scoring Ad--Query Pairs}
In this case study, we demonstrate the potential of the calibration functions. The goal is to score how well an ad matches a web search query, based on five different features that each measure a different notion of a good match. The score is required to be monotonic with respect to all five features. The labels are binary, so this is trained and tested as a classification problem. The train and test sets were independently and identically distributed, and are detailed in Table \ref{tab:datasetSummary}. 

Results are shown in Table \ref{tab:adQueryResults}. The cross-validated lattice size was $2 \times 2 \times 2 \times 2 \times 2$, and the calibration functions each used 5 changepoints. Removing the calibration functions and re-cross-validating the lattice size resulted in a larger lattice sized  $4 \times 4 \times 4 \times 4 \times 4$, and slightly worse (but not statistically significantly worse) accuracy.  In total, the uncalibrated lattice model used $1024$ parameters, whereas the calibrated lattice model used only $57$ parameters. We hypothesize that the smaller calibrated lattice will be more robust to feature noise and drift in the test sample distribution than the larger uncalibrated lattice model. In general, we find that the one-dimensional calibration functions are a very efficient way to capture the flexibility needed, and that in conjunction with good one-dimensional calibrations, only coarse-grained (e.g. $2^D$) lattices are needed.  

Both with and without calibration functions, the lattice regression models were statistically significantly better than the linear model. The random forest performed well, but was not statistically significantly better than the lattice regression. 

A boosted stumps model was also trained for this problem. See Fig. \ref{fig:StumpsVsLattice}  for a comparison of two-dimensional slices of the boosted stumps and lattice functions. The boosted stumps'  test set accuracy was relatively low at $75.4\%$. In practice, the goal of this problem is to have a score useful for ranking candidates as well as determining if they are a sufficiently good match. Even with many trees, this model produces many ties due its piecewise-constant surface.  In addition, the live experiments with the boosted stumps showed that the output was problematically sensitive to feature noise, which would cause samples near the boundary of two piecewise constant surfaces to experience fluctuating scores. 

\begin{table} 
\begin{center}
\begin{tabular}{lrc}
& Test Set Accuracy & Monotonic Guarantee?  \\ \hline
Linear Model & 77.2\% & yes \\  
Random Forests & 78.8\% & no \\
Lattice Regression & 78.7\% & yes \\ 
... Remove Continuous Signal Calibration & 78.4\%  & yes \\ 
\end{tabular}
\end{center}
\caption{Comparison on an ad-query scoring problem.} \label{tab:adQueryResults}
\end{table}

\begin{figure}[t]
\begin{center}
\begin{tabular}{cc}
\includegraphics[width=1.5in,height=1.5in]{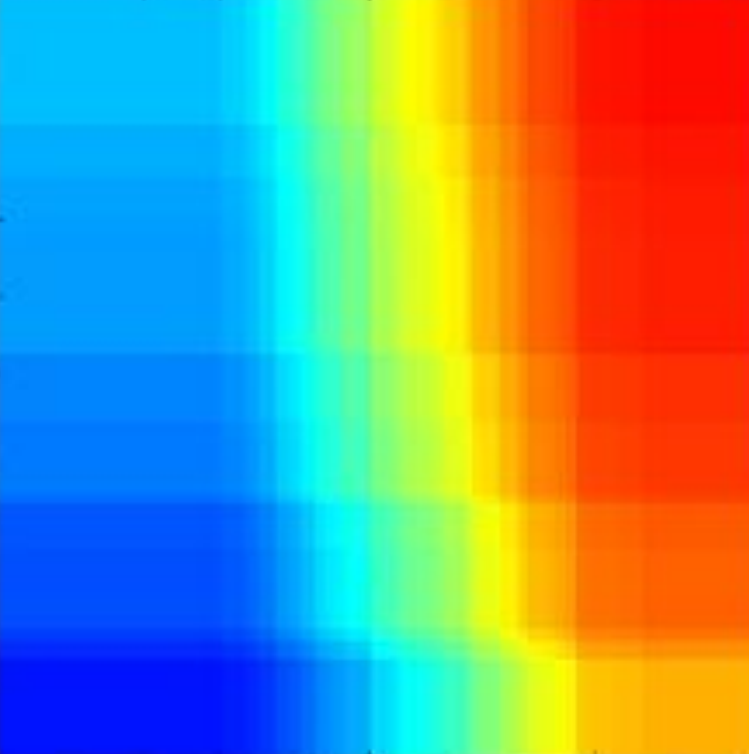} & \includegraphics[width=1.5in,height=1.5in]{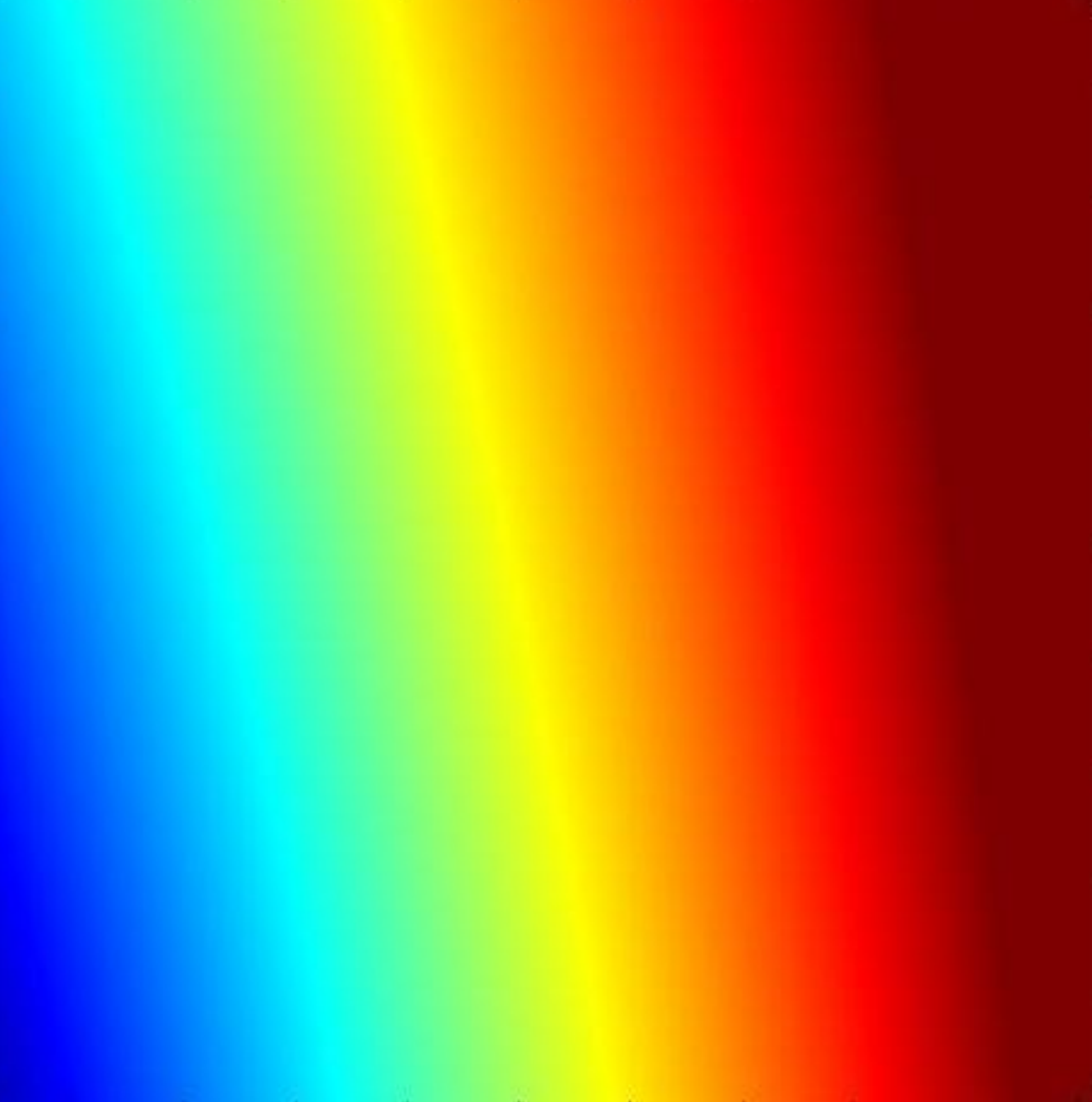}\\
Boosted Stumps & Lattice Regression
\end{tabular}
\end{center}
\caption{Slices of the learned ad-query matching functions for boosted stumps and a $2\times 2 \times 2 \times 2 \times 2$ lattice regression, plotted as a function of two of the five features, with median values chosen for the other three features. The boosted stumps required hundreds of stumps to approximate the function, and the resulting function is piecewise constant, creating frequent ties when ranking a large number of ads for a given query, despite a priori knowledge that the output should be strictly monotonic in each of the features.}
\label{fig:StumpsVsLattice}
\end{figure}

\subsection{Case Study: Rendering Classifier}
This case study demonstrates training a flexible function (using a lattice) that is monotonic with respect to fifteen features. The goal is to score whether a particular display element should be rendered on a webpage. The score is required to be monotonic in fifteen of the features, and there is a sixteenth Boolean feature that is not constrained.  The training and test sets (detailed in Table \ref{tab:datasetSummary}) consisted almost entirely of samples known to be difficult to correctly classify (hence the rather low accuracies).

We used a fixed $2^{16}$ lattice size, a fixed $5$ changepoints per feature for the six continuous signals (the other ten signals were Boolean), and no graph regularization, so no hyperparameters were optimized for this case study.  Simplex interpolation was used for speed.  A single training loop through the 20,000 training samples took around five minutes on a Xeon-type Intel desktop using a single-threaded C++ implementation with sparse vectors, with the training time dominated by the constraint handling.  Training in total took around five hours. 

Results in Table \ref{tab:renderingResults} show substantial gains over the linear model, while still producing a monotonic, smooth function. The lattice regression was also statistically significantly better than random forests, we hypothesize due to the regularization provided by the monotonicity constraints which is important in this case due to the difficulty of the problem on the given examples and the relatively small number of training samples.

\begin{table}
\begin{center}
\begin{tabular}{lrc}
& Test Set Accuracy & Monotonic Guarantee?  \\ \hline
Linear Model & 54.6\% & yes \\
Random Forest & 61.3\% & no \\
Lattice Regression & 63.0\% & yes \\
\end{tabular}
\end{center}
\caption{Comparison on a rendering classifier.} \label{tab:renderingResults}
\end{table}

\subsection{Case Study: Fusing Pipelines}
While this paper focuses on learning monotonic functions, we believe it is also the first paper to propose applying lattice regression to classification problems, rather than only regression problems. With that in mind, we include this case study demonstrating that lattice regression \emph{without constraints} also performs similarly to random forests on a real-world large-scale multi-class problem. 

The goal in this case study is to fuse the predictions from two pipelines, each of which makes a prediction about the likelihood of seven user categories based on a different set of high-dimensional features.  Because each pipeline's probability estimates sum to one, only the first six probability estimates from each pipeline are needed as features to the fusion, for a total of twelve features. The training and test set were split by time, with the older 1.6 million samples used for training, and the newest 390,000 samples used as a  test set. 

The lattice was trained with a multi-class logistic loss, and used simplex interpolation for speed. The cross-validated model was a $2^{12}$ lattice for six of the output classes (with the probability of the seventh class being subtracted from one) and no calibration functions, resulting in a total of $2^{12} \times 6 = 24,576$ parameters.  

The results are reported in Table \ref{tab:ageResults}. Even though Pipeline 2 alone is $6.5\%$ more accurate than Pipeline 1 alone, the test set accuracy can be increased by fusing the estimates from both pipelines, with a small  improvement in accuracy by lattice regression over the random forest classifier, logistic regression, or simply averaging the two pipeline estimates.  

\begin{table}
\begin{center}
\begin{tabular}{lrc}
& Test Set Accuracy Gain\\
&  on top of Pipeline 1 Accuracy  \\ \hline
Pipeline 2 Only & 6.5\% \\
Average the Two Pipeline Estimates & 7.4\%  \\
Fuse with Linear Model & 8.5\%  \\
Fuse with Random Forest & 9.3\% \\
Fuse with Lattice Regression & 9.7\%   \\
\end{tabular}
\end{center}
\caption{Comparison on fusing user category prediction pipelines.} \label{tab:ageResults}
\end{table}

\subsection{Case Study: Video Ranking and Large-Scale Learning}\label{sec:youtube}
This case study demonstrates large-scale training of a large monotonic lattice and learning from ranked pairs.  The goal is to learn a function to rank videos a user might like to watch, based on the video they have just watched.  Experiments  were performed on anonymized data from YouTube. 

Each feature vector $x_i$ is a vector of features about a pair of videos, $x_i = h(v_j, v_k)$, where $v_j$ is the watched video, $v_k$ is a candidate video to watch next, and $h$ is a function that takes a pair of videos and outputs a twelve-dimensional feature vector $x_i$. For example, a feature might be the number of times that video $v_j$ and video $v_k$ were watched in the same session.

Each of the twelve features was specified to be positively correlated with users viewing preference, and thus we constrained the model to be monotonically increasing with respect to each.  Of course, human preference is complicated and these monotonicity constraints cannot fully model human judgement.  For example, knowing that a video that has been watched many times is generally a very good indicator that it is good to suggest, and yet a very popular video at some point will flare out and become less popular.

Monotonicity constraints can also be useful to enforce secondary objectives.  For example, all other features equal, one might prefer to serve fresher videos. While users in the long-run want to see fresh videos, they may preferentially click on familiar videos, thus click data may not capture this desire.  This secondary goal can be enforced by constraining the learned function to be monotonic in a feature that measures video freshness. This achieves a multi-objective function without overly-complicating or distorting the training label definition.

There are billions of videos in YouTube, and thus many many pairs of watched-and-candidate videos to score and re-score as the underlying feature values change over time. Thus it is important the learned ranking functions to be cheap to evaluate, and so we use simplex interpolation for its evaluation speed; see Section \ref{sec:runtimes} for comparison of evaluation speeds.


We trained to minimize the ranked pairs objective from (\ref{eqn:latticeRegressionRanking}), such that the learned function $f$ is trained for the goal of minimizing pairwise ranking errors,
\begin{equation*}
f(h(v_j, v_k^+)) > f(h(v_j, v_k^-)),
\end{equation*}
for each training event consisting of a watched video $v_j$, and a pair of candidate videos $v_k^+$ and $v_k^-$ where there is information that a user who has just watched video $v_j$ prefers to watch $v_k^+$ next  over $v_k^-$.

\subsubsection{Which Pairs of Candidate Videos?}
A key question is which sample pairs of candidate videos  $v_k^+$ and $v_k^-$ should be used as the preferred and unpreferred videos for a given watched video $v_j$. We used anonymized click data from YouTube's current video-suggestion system. For each watched video $v_j$, if a user clicked a suggested video in the second position or below, then we took the clicked video as the preferred video $v_k^+$, and the video suggested right above the clicked video as the unpreferred video $v_j^-$.  We call this choice of  $v_k^+$ and $v_k^-$ a \emph{bottom-clicked pair}. This choice is consistent with the findings of \citet{Joachims:2005}, whose eye-tracking experiments on webpage search results showed that users on average look at least at one result above the clicked result, and that these pairs of preferred/unpreferred samples correlated strongly with explicit relevance judgements. Also, using bottom-clicked pairs removes the \emph{trust bias} that users know they are being presented with a ranked list and prefer samples that are ranked-higher \citep{Joachims:2005}. In a set of preliminary experiments, we also tried training using either a randomly sampled video as $v_k^-$, or the video just after the clicked video, and then tested on bottom-clicked pairs. Those results showed test accuracy on bottom-clicked pairs was up to $1\%$ more accurate if the training set only included the bottom-clicked pairs, even though that meant fewer training pairs.

An additional goal (and one that is common in commercial large-scale machine learning systems for various practical reasons) is for the learned ranking function to be as similar to the current ranking function as possible.  That is, we wish to minimize changes to the current scoring if they do not improve accuracy; such accuracy-neutral changes are referred to as \emph{churn}. To reduce churn, we added in additional pairs that reflect the decisions of the current ranking function. Each of these pairs also takes the clicked video as the preferred $v_k^+$, but sets the unpreferred video $v_k^-$ to be the video that the current system ranked ten candidates lower than the clicked video.  The dataset is a 50-50 mix of these churn-reducing pairs and bottom-clicked pairs.

\subsubsection{More Experimental Details}
The dataset was randomly split into mutually exclusive training, test, and validation sets of size 400 million, 25 million, and 25 million pairs, respectively. To ensure privacy, the dataset only contained the feature vector, and no information identifying the video or user. The disadvantage of that is the train, test and validation sets are likely to have some samples from the same videos and same users.  However, in total the datasets capture millions of unique users and unique watched videos.

We used a fixed $3^{12}$ lattice, for a total of 531,441 parameters. The pre-processing functions were fixed in this case, so no calibration functions were learned. We compared training on  increasingly-larger randomly-sampled subsets of the 400 million training set (see Figure \ref{fig:youtube} for training set sizes).   We compared training on a single worker to the parallelize-and-average strategy explained in Section \ref{sec:parallelize}. Parallel results were parallelized over 100 workers.  The stepsize was chosen independently for each training set based on accuracy on the validation set. 

We report results with and without monotonicity constraints. For the unconstrained results, each training (single or parallel) touched each sample in the training set once. For the  monotonic results (single or parallelized), each sample was touched ten times, and minibatching was used with a minibatch size of 32 stochastic gradients. Logistic loss was used. 

\subsubsection{Results}
Figure \ref{fig:youtube} compares test set accuracy for single and parallelized training for different amounts of training data, with and without monotonicity constraints.  For each dataset, the single and parallel training saw the same total number of training samples and were allowed the same total number of stochastic gradient updates.

On the click data test set, not using monotonicity constraints (the dark lines) is about  $.5\%$ better at pairwise accuracy than if we constrain the function to be monotonic. However, in live experiments that required ranking all videos (not only ones that had been top-ranked in the past, and hence possibly clicked on), models trained with monotonicity constraints showed better performance on the actual measures of user-engagement (as opposed to the training metric of pairwise accuracy). This discrepancy appears to be due to the biased sampling of the click data we train (and test on offline), as the click-data has a biased distribution over the feature space compared to the distribution of all videos which must get ranked in reality.  The biased distribution of the click data appears to cause parameters in sparser regions of the feature space to be non-monotonic in an effort to increase the flexibility (and accuracy) of the function in the denser regions, thus increasing the accuracy on the click data. Enforcing monotonicity helps address this sampling bias problem by not allowing the training to ignore the accuracy in sparser regions that are important in practice to accurately rank all videos.

Even though there are 500k parameters to train, the click-data accuracy is already very good with only 500k training samples, and test accuracy increases only slightly when trained on 400 million samples compared to 10 million samples. This is largely because the click-data samples are densely clustered in the feature space, and with simplex interpolation, only a small fraction of the 500k parameters control the function over the dense part of the feature space. 

The darker lines of Figure \ref{fig:youtube} show the parallelization versus single-machine results \emph{without} monotonicity constraints.  Unconstrained, the parallelized runs appear to perform slightly better to the single-machine training given the same number of training samples (and the same total number of gradient updates). We hypothesize this slight improvement is due to some noise-averaging across the 100 parallelized trained lattices. 

\begin{figure}[ht!]
\begin{center}
\begin{tabular}{l}
\includegraphics[width=1.05\textwidth]{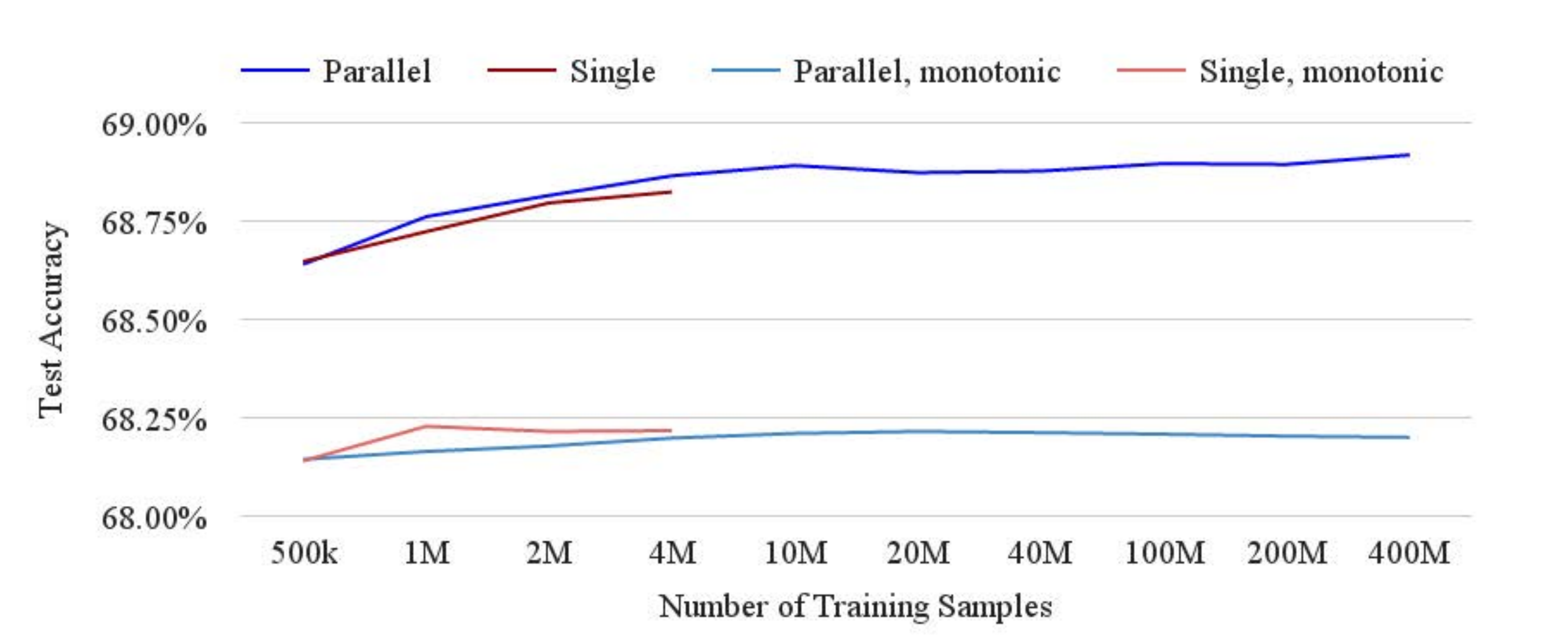} 
\end{tabular}
\end{center}
\caption{Comparison of training with a single worker versus 100 workers in parallel, as a function of training set size. } \label{fig:youtube}
\end{figure}

The lighter lines of Figure \ref{fig:youtube}  show the parallelization versus single-machine results \emph{with} monotonicity constraints.  Trained on 500k pairs, the parallelized training and single-machine monotonic  training produce the same test accuracy. However, as the training set size increases, the parallelized training takes more data to achieve the same accuracy as the single-machine training. We believe this is because averaging the 100 monotonic lattices is a convex combination of lattices likely on the edge of the monotonicity constraint set, producing an average lattice in the interior of the constraint set, that is, the averaged lattice is over-constrained. 

\subsection{Run Times} \label{sec:runtimes}
We give some timing examples for the different interpolations and for training. 

Figure \ref{fig:evalTime} shows average evaluation times for multilinear and simplex interpolation of one sample from a $2^D$ lattice for $D = 4$ to $D = 20$ using a single-threaded 3.5GHz Intel Ivy Bridge processor. Note the multilinear evaluation times are reported on a log-scale, and on a log scale the evaluation time increases roughly linearly in $D$, matching the theoretical $O(2^D)$ complexity given in Section \ref{sec:fast}.  The simplex evaluation times scale roughly linearly with $D$, consistent with the theoretical $O(D \log D)$ complexity. For $D=6$ features, simplex interpolation is already three times faster than multilinear. With $D=20$ features, the simplex interpolation is still only 750 nanoseconds, but the multilinear interpolation is about $15,000$ times slower, at around 12 milliseconds.

Training times are difficult to report in an accurate or meaningful way due to the high-variance of running on a large, shared, distributed cluster. Here is one example: with every feature constrained to be monotonic, a single worker training one loop of a $2^{12}$ lattice on 4 million samples usually takes around 15 minutes, whereas with 100 parallelized workers one loop through 400 million samples (4 million samples for each worker) usually takes around 20 minutes. Large step-sizes can take much longer than smaller stepsizes, because larger updates tend to violate more monotonicity constraints and thus require more expensive projections.  Minibatching is particularly effective at speeding up training because the averaged batch of stochastic gradients reduces the number of monotonicity violations and the need for projections. Without monotonicity constraints, training is generally $10 \times$ to $1000 \times$ faster, depending on how non-monotonic the data is. 

\begin{figure}[h!]
\begin{center}
\begin{tabular}{cc}
\includegraphics[width=.5\textwidth]{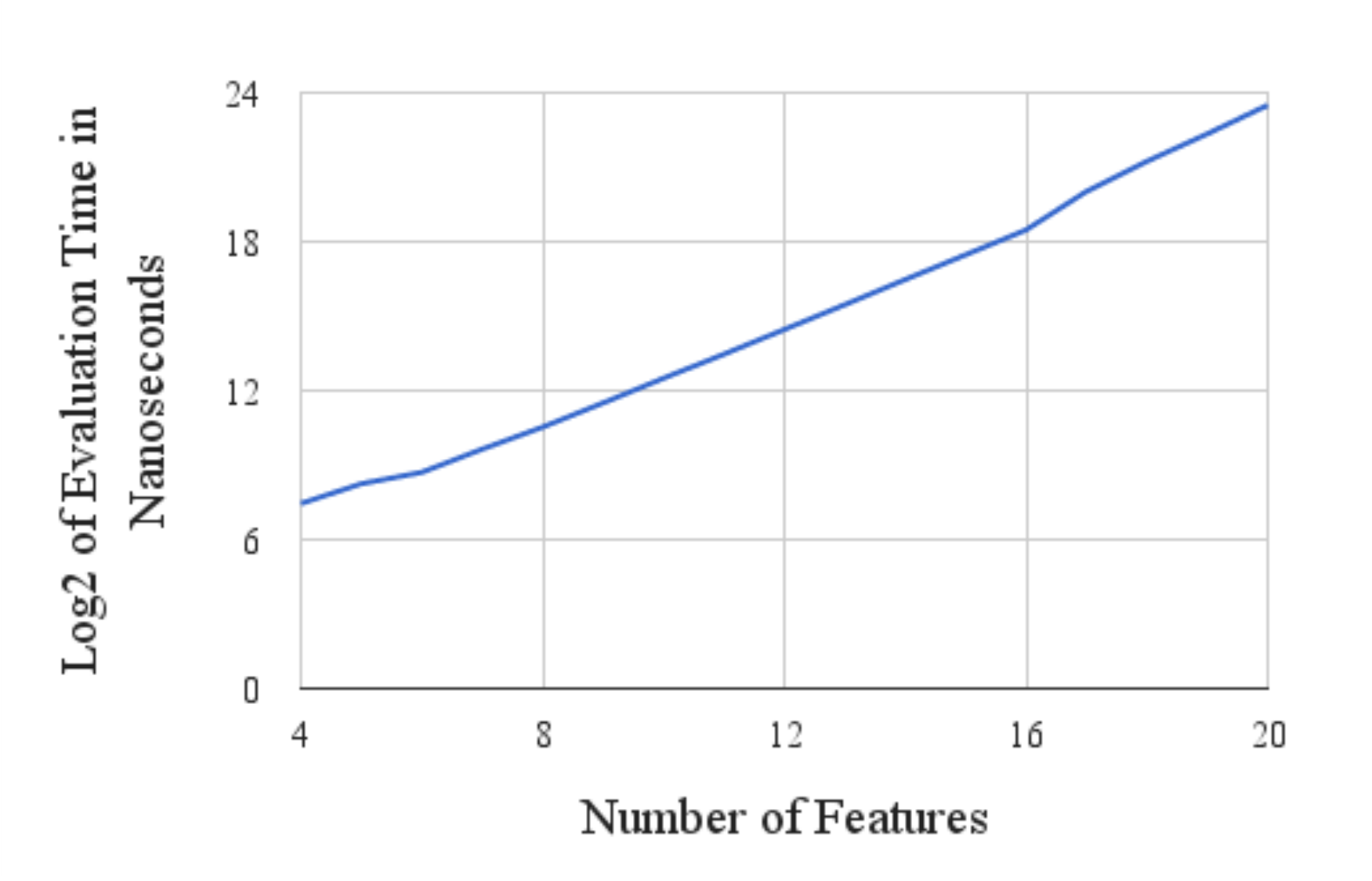} & \includegraphics[width=.5\textwidth]{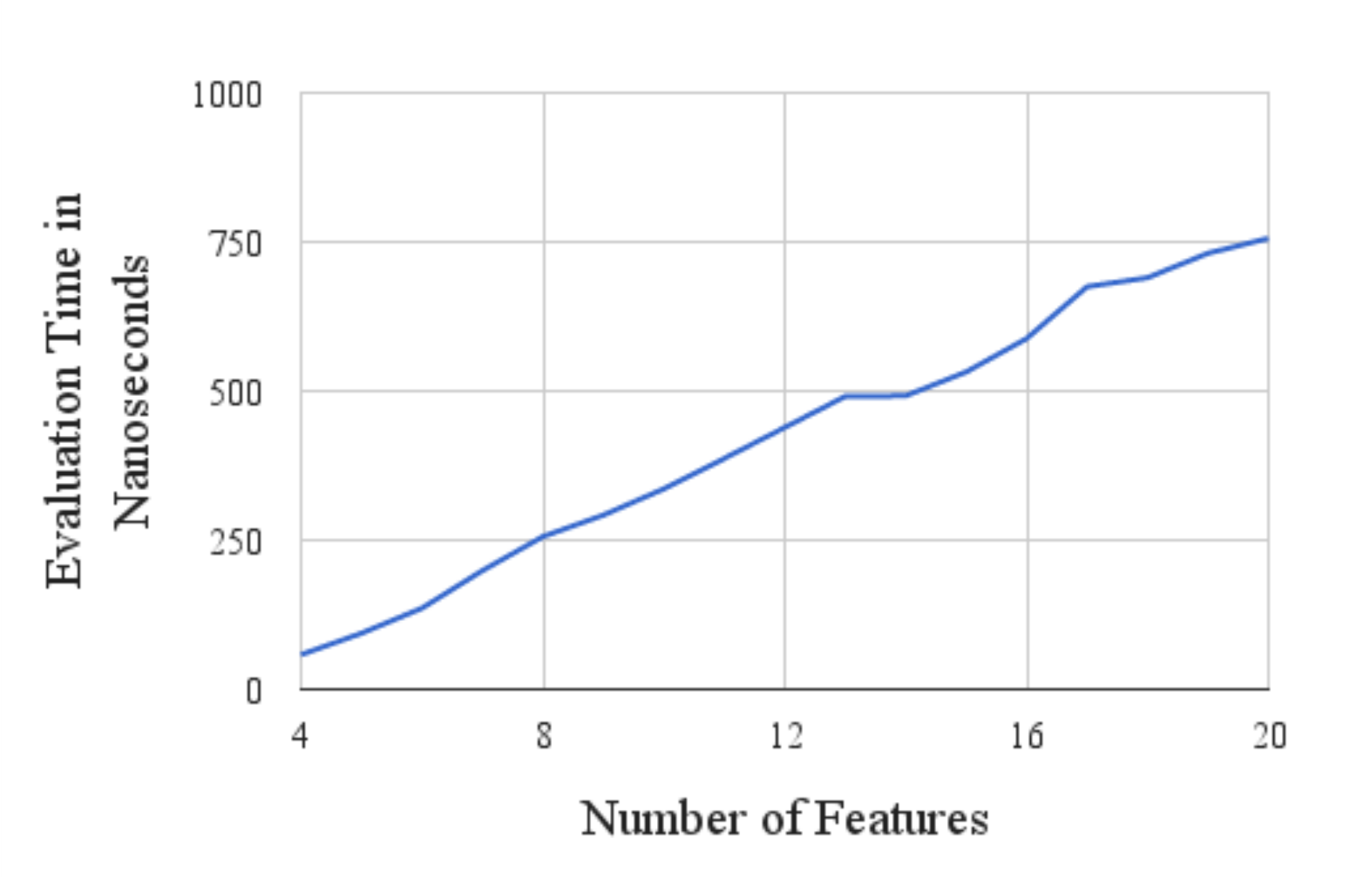} \\
(a) Multilinear Interpolation & (b) Simplex Interpolation  \\
\end{tabular}
\end{center}
\caption{Average evaluation time to interpolate a sample from a $2^D$ lattice. Figure (a) shows the multilinear interpolation time on a $\log_2$ scale in nanoseconds. Figure (b) shows the much faster simplex interpolation time in nanoseconds. Simplex interpolation is $10\times$ faster than multilinear for $D=9$ features, about $100\times$ faster for $D=13$ features,  and over $1,000\times$ faster for $D=17$ features. } \label{fig:evalTime}
\end{figure}

\section{Discussion and Some Open Questions} \label{sec:discussion}
We have proposed an approach to effectively learn flexible, monotonic functions for low-dimensional machine learning problems of classification, ranking, and regression. We addressed a number of practical issues, including interpretability, evaluation speed, automated pre-processing of features, missing data, and categorical features. Experimental results show statistically significant state-of-the-art performance on the largest training sets and largest number of features published for monotonic methods. 

Practical experience has shown us that being able to check and ensure monotonicity helps users trust the model, and leads to models that generalize better. For us, the monotonicity constraints have come from engineers who believe the output should be monotonic in the feature.  In the absence of clear prior information about monotonicity, it may be tempting to use the direction of a linear fit to specify a monotonic direction and then use monotonicity as a regularizer.  \citet{Sill:08} point out that using the linear regression coefficients for this purpose can be misleading if features are correlated and not jointly Gaussian. 

For classifiers, requiring the function to be monotonic is a stronger requirement than needed to guarantee the decision boundary is monotonic.  We have seen in practice that this can occur: one trains an unconstrained lattice, find that it is non-monotonic, but that the thresholded function $g(x) = I_{f(x) > 0}$ \emph{is} monotonic. It is an open question how this could be enforced and whether it would be useful. 

One surprise was that for practical machine learning problems like those of Section \ref{sec:experiments}, we find a simple $2^D$ lattice is sufficient to capture the interactions of $D$ features, especially if we jointly optimize $D$ one-dimensional feature calibration functions. When we began this work, we expected to have to use much more fine-grained lattices with many vertices in each feature, or perhaps irregular lattices to achieve state-of-the-art accuracy. In fact, calibration functions can effectively linearize the data with respect to the label, making a $2^D$ lattice sufficiently flexible for most of the problems we have encountered.

For some cases, a $2^D$ lattice is too flexible. We reduced lattice flexibility with new regularizers: monotonicity, and the torsion regularizer that encourages a more linear model. While good for interpretability and accuracy, these regularization strategies do not reduce the model size. 

For a large number of features $D$, the exponential model size of a $2^D$ lattice is a memory issue. On a single machine, training and evaluating with a few million parameters is viable, but this still limits the approach to not much more than $D=20$ features. An open question is how such large models could be sparsified, and if useful sparsification approaches could also provide additional useful regularization.

A second surprise was that simplex interpolation provides similar accuracy to multilinear interpolation. The rotational dependence of simplex interpolation seemed at first troubling, but the proposed approach of aligning the shared axis of the simplices with the main increasing axis of the function appears to solve this problem in practice. The geometry of the simplices at first seemed odd in that it produces a locally linear surface over elongated simplices. However, this partitioning turns out to work well because it provides a very flexible piecewise linear decision boundary.  Lastly, we found that the theoretical $O(D \log D)$ computational complexity does result in orders of magnitude faster interpolation than multilinear interpolation as $D$ increases.

A common practical issue in machine learning is handling categorical data. We proposed to learn a mapping from mutually exclusive categories to feature values,  jointly with the other model parameters. We found categorical-mapping to be interpretable, flexible, and accurate. The proposed categorical mapping can be viewed as learning a one-dimensional embedding of the categories. Though we generally only needed two vertices in the lattice for continuous features, for categorical features we often find it helpful to use more vertices (a finer-grained lattice) for more flexibility. Some preliminary experiments learning two-dimensional embeddings of categories  (that is, mapping one category to $[0,1]^2$) showed promise, but we found this required more careful initialization and handling of the increased non-convexity.

Learning the monotonic lattice is a convex problem, but composing the lattice and the one-dimensional calibration functions creates a non-convex objective.  We used only one initialization of the lattice and calibrators for all our experiments, but tuned the stepsize of the stochastic gradient descent separately for the set of lattice parameters and the set of calibration parameters. In some cases we saw a substantial sensitivity of the accuracy to the initial SGD stepsizes. We hypothesize that this is caused by some interplay of the relative stepsizes and the relative size of the local optima.

We employed a number of strategies to speed up training. One of the biggest speed-ups comes from randomly sampling the additive terms of the graph regularizers, analogous to the random sampling of the additive terms of the empirical loss that SGD uses.  We showed that a parallelize-and-average strategy works for training the lattices.  The largest computational bottleneck remains the projections onto the monotonicity constraints. Mini-batching the samples reduces the number of projections and provides speed-ups, but a faster approach to optimization given possibly hundreds of thousands of constraints would be valuable.

While we focused on constraints that impose monotonicity, imposing other constraints on the lattice may be useful.  For example, one could learn a submodular function from noisy samples by imposing submodularity constraints on the parameters.  Or, to guarantee that change from a prior model is not too large (for churn reduction, or consistency), one can constrain the lattice parameters to be within a fixed margin of the prior model's prediction for the corresponding vertex.  If more than three (or more) vertices are used for a feature, then one can enforce Brooks'  Law \citep{Brooks:75}, in which the function is constrained to first increase but then decrease as one input is increased, all other inputs held constant. Brooks' Law is most famous for modeling the productivity of software teams as programmers are added - Brooks argued that at some point adding more programmers can make the project slower not faster, due to increased communication overhead and difficulties of achieving consensus on decisions, a phenomenon referred to as the \emph{the mythical man-month}. In fact, Brooks' Law may be useful for modeling a breadth of real-world relationships.  For example, an extraordinarily large number of watches of a video suggests that everyone has already seen it, and that it may not be as good a recommendation as a video with slightly fewer watches. Or if predicting a used car's value, one expects the value to decrease with the car's age up to a certain point, after which it becomes a \emph{classic} car and (all else held constant), its value increases with age. Combined with the proposed calibration functions that determine how to map raw values to the constrained lattice, it should be possible to learn effective functions with such complicated constraints for real-world problems. 


\section{Acknowledgments}
We thank Sugato Basu, David Cardoze, James Chen, Emmanuel Christophe, Brendan Collins,  Mahdi Milani Fard, James Muller, Biswanath Panda, and Alex Vodomerov for help with experiments and helpful discussions.

\bibliography{references}
\end{document}